\documentclass[tikz, 10pt]{amsart}
\usepackage[usenames,dvipsnames]{xcolor}
\usepackage{amsmath,amsthm,amssymb,color,comment,csquotes,enumerate,fancyhdr,filecontents,graphicx,verbatim}
\usepackage[round]{natbib}
\usepackage{pgfplots}
\usepackage{caption}
\usepackage{bbm} 

\usepackage{pgfplots}
\pgfplotsset{compat=1.18}

\usepackage{hyperref}
\hypersetup{colorlinks=true,linkcolor=MidnightBlue,citecolor=MidnightBlue,bookmarks=false,backref=page}

\makeatletter
\newcommand{\customlabel}[2]{%
   \protected@write \@auxout {}{\string \newlabel {#1}{{#2}{\thepage}{#2}{#1}{}} }%
   \hypertarget{#1}{#2}
}
\makeatother

\newtheorem{thrm}{Theorem}
\theoremstyle{definition}

\newtheorem{defn}[thrm]{Definition}

\DeclareMathOperator*{\argmin}{arg\,min}
\DeclareMathOperator{\Expec}{\mathbf{E}}	

\begin{document}

\title[Statistical learning theory and Occam's razor]{Statistical learning theory and Occam's razor: \\ Regularization}
\author[Sterkenburg]{Tom F.\ Sterkenburg}
\date{June 12, 2026. This is a preliminary version. I welcome feedback.}
\address{Munich Center for Mathematical Philosophy (MCMP), LMU Munich \newline \indent Munich Center for Machine Learning (MCML)}
\email{tom.sterkenburg@lmu.de}

\begin{abstract}
The principle of Occam's razor, which instructs us to prefer simplicity in inductive inference, has attracted much scrutiny both in the philosophy of science and in machine learning. In either field, however, a justification for the principle has been elusive. In this paper, building on an earlier ``core argument,'' I spell out a justification from statistical learning theory for the procedure of regularization: for trading off fit for simplicity. The means-ends argument is that in order to profit from theoretical reliability and ``what-you-see-is-what-you-get'' guarantees, one must implement a certain preference for simplicity over fit. This is a genuine methodological justification, which neither collapses to a purely pragmatic principle that we prefer simplicity for its own sake, nor to an ontological assumption that the truth is simple.

\end{abstract}

\maketitle

\vspace{-8mm}

\section{Introduction}

The methodological role of simplicity in scientific theorizing is a long-standing topic in the philosophy of science \citep{Sob15,Bak22sep,DCGGHLMMNRSWWBS25pnas}. In  recent years, following the work of \citet{ForSob94bjps}, a considerable literature has spawned around the lessons we may draw from the role of simplicity in statistics, specifically from methods for model selection (\citealp{ZelKeuMca02}; \citealp[ch.\ 2]{Sob15}; \citealp[sect.\ 7]{Bak22sep}; \citealp[ch.\ 10]{SprHar19}).

The contemporary consensus is that these lessons are limited (\citealp{Kie97bjps}; \citealp{MyrHar02pos}; \citealp[chs.\ 6--7]{Nor21}). The way, say, the Akaike Information Criterion (AIC)  operates will not help us with, say, the rational reconstruction of a  preference for the Copernican to the Ptolemaic system. Even so, the methodology of statistics, including the role of simplicity, remains of independent interest to the philosophy of science, because of the importance of statistical inference to many branches of science \citep{Rom25sep}. The same holds, increasingly so, for machine learning.
 
In the field of machine learning, regular reference is made to Occam's razor, understood broadly as the methodological principle to prefer simplicity \citep{Mit97,DudHarSto01,ShaBen14,GooBenCou16}. These references and associated claims have drawn some critical attention in the philosophy of science \citep{HarKul07,Her20pos,BarCevGne22mam}. The focus of these works is statistical learning theory \citep{Vap00,ShaBen14}, arguably (still) the main mathematical framework for machine learning; and the  question is whether statistical learning theory can support a justification for Occam's razor. What these works illustrate is that to date there has been no successful attempt, neither in the philosophy nor in the machine learning literature, to spell out a justification of this kind.


This paper presents such a---\emph{hopefully}, successful---attempt. Building on a previous ``core argument'' \citep{Ste25mam}, I give a means-ends argument for the methodological norm to trade off empirical data fit with simplicity (what in machine learning is called  \emph{regularization}), in accordance with the learning rule of structural risk minimization (SRM).

The plan is as follows. I start, in section \ref{sec:core}, with introducing the main components of statistical learning theory. These are the ingredients for the core argument, which underpins a first methodological simplicity norm (``keep the model simple''). This norm, however, is limited in its applicability, which motivates, as I discuss in section \ref{sec:srm}, a more general mathematical perspective. This perspective yields the SRM learning rule, which implements explicit regularization and as such a second methodological simplicity norm (``trade fit for simplicity''). In order to spell out the justification for this norm, I take a closer look, in section \ref{sec:srmjust}, at the theoretical justification for the SRM rule. I here discuss the \emph{model-relativity} of theoretical justification, which poses a challenge to the possibility of a non-question-begging justification for a simplicity preference. I answer this challenge in section \ref{sec:arg}, where I give  the means-ends argument for a genuine methodological simplicity norm. I conclude in section \ref{sec:concl}.

In the course of developing my argument, I draw from various earlier insights in the philosophical literature on model selection; and the question may arise what is the point of ``supplementing the debate with a separate analysis of a method which only differs in mathematical details'' \citep[p.\ 772]{Kie01bjps}. My answer is, first, that the framework of statistical learning theory is perhaps not, as \citet[fn.\ 61]{Sob15} has it, ``dramatically'' different from that of the methods that took center stage in the earlier debate (essentially, AIC, and BIC---the Bayesian Information Criterion): but it is still different in important ways. What stands out, for instance, are the robust complexity notion of capacity, and the finite-sample bounds justifying the relevant learning rules. Second, the current work makes some progress in connecting related but independent strands in the (philosophy of) statistics and computer science literatures. Third, I do believe that my argument goes beyond what has been proposed in either literature so far. 

\citet[p.\ 41]{Kie97bjps} offers a meta-reflection on the current kind of project. He writes that the  aim of statistics is  to ``provide scientists with (better or worse) methods,'' whereas ``[p]hilosophy of science is concerned with the \emph{justification} of scientific practices.'' Since we cannot tell  ``whether the use of some given statistical method is justified in a given situation without describing the method and the background assumptions of the results on which it is based in a rigorous way,'' Kiesepp\"a concludes that, when dealing with particular statistical methods and results, the philosophy of statistics 

\begin{quote}
  must
be mathematically more rigorous than statistics itself. Otherwise it is not clear
why such philosophy should be pursued as a field distinct from both statistics
and popular science.
\end{quote}
I would not say that work in the philosophy of machine learning must strive to be more \emph{mathematically} rigorous than work in machine learning itself; the mathematics, after all, is the job of theoreticians of machine learning. However, I do think there is a job for philosophy in spelling out, with higher standards of rigor, the \emph{arguments} that \emph{use} the maths to justify methods or methodological principles.\footnote{Thus,
	to be clear, I do not---with one apparent exception: theorem \ref{thrm:ermfailnonunif} in the appendix---offer new mathematical results: the novelty of my contribution lies in laying down the argument.}
	That is, in any case, the spirit in which I proceed in this paper.

\section{The core argument}\label{sec:core}
Statistical learning theory counts as the 	``received view'' in the theoretical analysis of machine learning algorithms \citep[sec.\ 2]{GroGenSul24pc}.\footnote{The theory sprang out of the pioneering work of \cite{VapChe71tpa}; see \citep{Vap98,Vap00}. I will mainly follow the presentation by \citet{ShaBen14}, which is a synthesis of the Vapnik-Chervonenkis theory and Valiant's \citeyearpar{Val84acm} model of PAC learning.    } 
The most basic type of learning problem, and the exclusive focus of this paper, is binary classification (section \ref{ssec:binclass}). The crucial element in the theoretical analysis is the property of uniform convergence, which underwrites the learning rule of ERM (section \ref{ssec:unifconverm}). The fundamental theorem of statistical learning theory connects uniform convergence and successful learning by ERM with a notion of parsimony or simplicity of sets of classifiers, the VC dimension (section \ref{ssec:capfundthrm}). The core argument for a simplicity preference is assembled from these components (section \ref{ssec:corearg}).

\subsection{Binary classification}\label{ssec:binclass}
 As an ``imaginary and somewhat fanciful example'' \citep[pp.\ 1ff]{DudHarSto01}, imagine that, in the context of a wildlife preservation effort, you want to design a system to automatically keep a tally of the number of specimens of European seabass passing by in a certain coastal river. You set up an underwater camera at a particularly narrow strait, which automatically takes pictures of moving objects. Now the problem is to identify the pictures which indeed captured a seabass and not something else. 
 
You decide to use machine learning. You first inspect a number of pictures taken, and select features which are indicative of a seabass present. Say one feature you choose is the relative size of the main object in the picture, and another is the average brightness of the picture (you have software to calculate those). Then you collect a larger sample of pictures generated over several days, and painstakingly annotate those: seabass or no. 
Finally, you feed this training set of pictures to a learning algorithm, which will automatically infer or learn a general classifier, which will in turn be able to classify new pictures. What is a good learning algorithm?\footnote{In
	all of the following, the focus will  be on the learning or generalization step:  the algorithmic inference from training data to learned classifier. The toy example already shows that this focus is inherently limited. What features you pick is obviously important to the quality of the inference; and in reality there will be several iterations of and between stages of problem formulation, algorithmic learning, and evaluation. However, this paper is concerned with the theoretical justification for Occam's razor, and statistical learning theory is a theory of generalization.}

\subsubsection{The formal framework}
In statistical learning theory, problems of this type are formalized as follows. We have a domain $\mathcal{X}$ of \emph{instances}, which are usually themselves vectors of real-valued \emph{attributes}. In our example, the instances are the pictures, summarized in two real-valued features: size and brightness. If we assume these values are normalized to the unit interval, then the instances are  points in the space $[0,1]^2$. 
We further have a \emph{label} set $\mathcal{Y}$; in the example, there are  two labels, seabass or not (1 or 0). This is  an example of a \emph{binary classification} problem.

A \emph{hypothesis} is a particular classifier, a function $h: \mathcal{X} \rightarrow \mathcal{Y}$ from all possible instances to labels.\footnote{A 
	hypothesis is (more) often called a \emph{model} in machine learning, but due to risk of confusion with other uses of the term (inductive model, model selection) I will stick to hypothesis.} 
	A \emph{learning algorithm} $A$ is a function that receives a training \emph{sample} $S$, a finite ordered sequence of instance-label pairs, and returns a classifier.\footnote{Thus a learning algorithm, as defined here, is really a function. I abstract away in this paper from computational considerations.}

\subsubsection{The goal}\label{sssec:iidgoal}
An essential assumption in statistical learning theory is that training instances as well as new data instances are independently and identically distributed (i.i.d.)\ samples from some true but unknown distribution $\mathcal{D}$ over $\mathcal{X} \times \mathcal{Y}$.\footnote{In 
	the context of the problem of induction, the i.i.d.\ assumption already functions as a kind of ``uniformity of nature'' assumption (cf.\ \citealp[p.\ 29]{ForSob94bjps})---though, as we will see, the problem of induction still remains (cf.\ \citealp[fn.\ 3]{Str09abs}; \citealp{SteGru21syn}). } 
	This distribution thus describes both the selection of instances (new images) and the possibly noisy relationship between features and labels (between object size and average brightness, and depicting a seabass or not).\footnote{How
	well the theory applies to any particular learning problem depends, of course, on how plausible the i.i.d.\ assumption is. \citet[p.\ 586]{Gru07} writes that ``this is one of the few examples of a modeling assumption which may actually be quite realistic in some situations;'' \citet{Sha09abs} is more skeptical. 
Even if the assumption is reasonable for the learning process, the issue of \emph{distribution shift} in domains of application is a threat to the stability of the definition of true risk and as such to the robustness of the learned hypothesis \citep{GroGenSul24pc,FreGro23syn}.}  

This assumption allows us to define, for any given classifier $h$, the probability that it misclassifies a new, randomly generated instance, 
\begin{align}\label{eq:truerisk}
L_\mathcal{D}(h) := \mathbb{P}_{(X,Y) \sim \mathcal{D}}\left[h(X) \neq Y \right].    
\end{align}
This is the \emph{true risk} of $h$, and the goal of learning is to find a classifier which minimizes the true risk. Crucial in the below analysis of what makes for a good learning algorithm will be the gap between an hypothesis' true risk and its average or \emph{empirical error} on the training sample $S$,
\begin{align}\label{eq:emperror}
L_S(h) := \frac{|\{(x,y)\in S: h(x) \neq y \}|}{|S|}.  
\end{align}

\subsubsection{Pattern recognition, curve fitting, model selection}\label{sssec:curvefit}
The problem of classification is not quite the same as the problem of curve fitting usually considered in the philosophical literature (\citealp[ch.\ VIII]{Gly80}; \citealp{ForSob94bjps}). There we assume a set of noisy instances $(x_i,y_i)$ of a functional relationship $f(x)=y$, and the goal is to infer the  curve $f$. This kind of problem is called \emph{regression} in machine learning.

For our purposes, however, more important are the structural similarities between these types of (in early machine learning jargon) ``pattern recognition'' problems. 
First, at a formal level, a binary classifier can be seen to select a subset of the instances (those labeled 1); in our example, a subset of the unit square. A type of classifiers are the \emph{separators}: in our example, curves which cut the square into two parts. For instance, a linear separator is a line through the square, with the instances falling on one side of the line classified positively. A curve fitting problem is thus similar to a type of classification problem, where a classifier separates instances under the curve from instances above the curve (cf.\ \citealp[p.\ 65]{HarKul07}). 

Second, the original curve fitting problem, and generally problems in statistical \emph{model selection}, are usually seen as a two-step procedure. First, we select a model or a family of curves (e.g., the quadratic curves); second, we select a particular curve from the family (e.g., the least-squared-error quadratic). In the classification problem, we also discern these two steps.

\subsubsection{Hypothesis classes}
The first step is the specification of a family or class $\mathcal{H}$ of hypotheses. For instance, we could choose the $\mathcal{H}^\mathrm{pol}_1$ of linear separators, the class $\mathcal{H}^\mathrm{pol}_2$ of quadratic separators, or even the class $\mathcal{H}^\mathrm{pol}$ of all polynomial separators. 

Having chosen a hypothesis class $\mathcal{H}$ ourselves, we now want to delegate the second step to the learning algorithm: to learn, from the training data, a hypothesis with minimal true risk among those in $\mathcal{H}$. 
As we will see, the analysis of what is a good learning algorithm also has ramifications for how to make our initial choice.

\subsection{Uniform convergence and ERM}\label{ssec:unifconverm}
The i.i.d.\ assumption (sect.\ \ref{sssec:iidgoal}) allows us to bring in the law of large numbers, guaranteeing that, for any fixed hypothesis $h$, the empirical error \eqref{eq:emperror} of $h$ will,  as we draw larger and larger samples,  converge  in probability to its true risk \eqref{eq:truerisk}. However, we are not interested in a fixed hypothesis. We are interested in the performance of a learning algorithm, which, depending on the data, can select different hypotheses from chosen hypothesis class $\mathcal{H}$. For this we need something stronger, namely a ``uniform law of large numbers,'' which bounds the difference between empirical errors and true risks of all hypotheses \emph{uniformly}. 


\begin{defn}[Uniform convergence]\label{thrm:uniconv}
Hypothesis class $\mathcal{H}$ has the uniform convergence property if there exists a sample complexity function $m_\mathcal{H}^\mathrm{uc}: (0,1)^2 \rightarrow \mathbb{N}$  such that for all $\epsilon,\delta \in (0,1)$ and for any $\mathcal{D}$, we have for $m \geq m_\mathcal{H}^\mathrm{uc}(\epsilon,\delta)$ that  
\begin{align}\label{eq:unifconv0}
\mathbb{P}_{S \sim \mathcal{D}^m}\left[(\forall h \in \mathcal{H}) \left[ | L_\mathcal{D}(h) - L_{S}(h) | \leq \epsilon\right]\right] \geq 1-\delta.
\end{align}
\end{defn}

\subsubsection{Wysiwyg}

The uniform convergence property gives a ``what-you-see-is-what-you-get'' (``wysiwyg'') bound: for large enough sample size, with high probability, the empirical error of each $h$ (which can be ascertained from the training data---what you see) is  indicative (up to error term $\epsilon$)  of its true risk (what you get):
\begin{align}\label{eq:wysiwyg1}
(\forall h \in \mathcal{H}) \left[  | L_\mathcal{D}(h) - L_S(h)| \leq \epsilon \right].
\end{align}

This property motivates and justifies a basic learning rule.

\subsubsection{Empirical risk minimization}\label{sssec:ermdef}
To draw advice from the uniform convergence bound about what hypothesis a good learning algorithm should return for any given data sample $S$, it is helpful to rewrite it as a bound for any given $m$. To that end, we define the error function $\epsilon_\mathcal{H}: \mathbb{N} \times (0,1) \rightarrow (0,1)$ by
\begin{align}
\epsilon_\mathcal{H}(m,\delta) = \min\{\epsilon \in (0,1): m \geq m_\mathcal{H}^\mathrm{uc}(\epsilon,\delta) \},
\end{align}
that is, as giving the smallest error $\epsilon$ such that an $m$-length sample will with specified high probability satisfy the $\epsilon$-wysiwyg property \eqref{eq:wysiwyg1},
\begin{align}\label{eq:wysiwyg2}
(\forall h \in \mathcal{H}) \left[  | L_\mathcal{D}(h) - L_S(h)| \leq \epsilon_\mathcal{H}(m,\delta) \right].
\end{align}

This bound  directly implies that, for all $h \in \mathcal{H}$ simultaneously,
\begin{align}\label{eq:wysiwyg3} L_\mathcal{D}(h) \leq L_S(h) +  \epsilon_\mathcal{H}^\mathrm{uc}(m,\delta).
\end{align}

Vapnik's first ``inductive principle'' \citeyearpar[p.\ 20]{Vap00}, the learning rule of empirical risk minimization (ERM), is defined by selecting a hypothesis in $\mathcal{H}$ that minimizes this bound on the true error.\footnote{Note
	that the ERM rule is not quite an algorithm, as defined before, because the definition leaves underdetermined how it breaks ties in case of multiple error-minimizing hypotheses. We can ignore this in the current paper, because the theoretical analysis holds for any specific implementation of the ERM rule. I also ignore the challenge of actually implementing (approximations to) ERM, which is a problem of \emph{optimization} (see, e.g., \citealp[ch.\ 5]{HarRec22}).}
	Since the error term is identical for all hypotheses in $\mathcal{H}$, this comes down to picking a hypothesis with minimal empirical error.

\begin{defn}[ERM]
Given hypothesis class $\mathcal{H}$ (with the uniform convergence property). The ERM rule for $\mathcal{H}$ is defined by
\begin{align}\label{eq:erm}
\mathrm{ERM}_\mathcal{H}(S) \in \argmin_{h \in \mathcal{H}} L_S(h) .
\end{align}
\end{defn}

\subsubsection{ERM's theoretical justification (1)}\label{sssec:ermjust1}
If all hypotheses' training errors are good indication of their true risks, then, in order to obtain an hypothesis with minimal true risk, it makes sense to select an hypothesis with minimal training error. 

Formally, by the bound \eqref{eq:wysiwyg2} or \eqref{eq:wysiwyg3} applied to the particular hypothesis selected by ERM, we first have the wysiwyg bound
\begin{align}\label{eq:ermwysiwyg}
 L_\mathcal{D}(\mathrm{ERM}_\mathcal{H}(S)) \leq L_S(\mathrm{ERM}_\mathcal{H}(S)) + \epsilon_\mathcal{H}^\mathrm{uc}(m,\delta).
\end{align}
Next, by the definition of ERM as an empirical error minimizer, and another application of the wysiwyg bound \eqref{eq:wysiwyg2},
\begin{align}
L_S(\mathrm{ERM}_\mathcal{H}(S)) + \epsilon_\mathcal{H}^\mathrm{uc}(m,\delta) &\leq \min_{h \in \mathcal{H}} L_S(h) + \epsilon_\mathcal{H}^\mathrm{uc}(m,\delta) \\
&\leq \min_{h \in \mathcal{H}} L_\mathcal{D}(h) + 2\epsilon_\mathcal{H}^\mathrm{uc}(m,\delta),
\end{align}
so that, with \eqref{eq:ermwysiwyg}, 
\begin{align}\label{eq:ermlearn}
 L_\mathcal{D}(\mathrm{ERM}_\mathcal{H}(S)) \leq \min_{h \in \mathcal{H}} L_\mathcal{D}(h) + 2\epsilon_\mathcal{H}^\mathrm{uc}(m,\delta).
\end{align}

This bound gives a justification for using ERM, as selecting (with high probability, up to a known and constant error term) the best hypothesis in the class.

\subsubsection{Learnability and reliability}\label{sssec:pac}
The notion of \emph{PAC} (``probably approximately correct'') \emph{learnability} that is central in \citet{ShaBen14}'s presentation follows from retranslating this bound again, fixing error instead of sample size. Namely, hypothesis class $\mathcal{H}$ is PAC learnable by ERM$_\mathcal{H}$ if there is a sample complexity function $m_\mathcal{H}: (0,1)^2  \rightarrow \mathbb{N}$ such that for every $\epsilon,\delta \in (0,1)$, every $\mathcal{D}$, 
we have for $m \geq 	m_\mathcal{H}(\epsilon,\delta)$ that for every $h \in \mathcal{H}$,
\begin{align}\label{eq:pacbound}
\mathbb{P}_{S \sim \mathcal{D}^m}\left[ L_\mathcal{D}(\mathrm{ERM}_\mathcal{H}(S)) \leq \min_{h \in \mathcal{H}} L_\mathcal{D}(h) + \epsilon \right] \geq 1-\delta.
\end{align}

We could therefore also refer to bound \eqref{eq:ermlearn} as a learnability bound. I will rather follow \citet{HarKul07} in employing the general label of  \emph{reliability}.

\subsubsection{ERM's theoretical justification (2)}\label{sssec:ermjust2}

A justification for ERM is thus this reliability bound \eqref{eq:ermlearn} of probably selecting the approximately best hypothesis in the class. But the wysiwyg bound \eqref{eq:ermwysiwyg} is also in itself important, because it means that probably the training error will be an indication of how good approximately this selected hypothesis actually is. That means that if our hypothesis class is not good, we will probably be able to tell from the training error, and act accordingly  (cf.\ \citealp[sect.\ 11.3]{ShaBen14}).\footnote{The wysiwyg justification is not emphasized in \citep{Ste25mam}, but it will  be important for the argument from SRM.}
 
 Caution is warranted, though, in interpreting such frequentist guarantees. A pre-sampling wysiwyg guarantee that training error will probably be close to true error does not yet entail that, after sampling and the selection of a particular hypothesis, we can infer from high (low) training error of this hypothesis that its true error is probably high (low), too. (This can be seen most directly from the fact that the probability in the latter assertion must be interpreted epistemically.) Such a conclusion, or a corresponding decision, would need an additional reasoning step, like a Fisherian disjunctive or a Neymanian behaviorial argument.\footnote{See,
 	e.g., \citep{Spr16inc}. Fisher's disjunction, in this context, would be that either training and true error are close, \emph{or} something exceptional has happened; since it is reasonable to rule out the latter, we can conclude the former. The behavioral argument would be that if we always decide to act correspondingly (e.g., discard the selected hypothesis and start over with a different class if training error is high), then in the long run we would only go astray a fraction of the time. } 
	The same holds for the reliability guarantee that ERM probably finds the near-best hypothesis in the class. More could be said about this issue, but in this paper I will simply accept that frequentist guarantees  provide a justification.\footnote{In 
	machine learning practice, one would normally further seek to corroborate that a selected hypothesis is  good by assessing it on an independent test set. 
Since we then have a single fixed hypothesis, we can from the usual law of large numbers (more specifically, Hoeffding's equality) derive a tight wysiwyg bound on the probable difference between the hypothesis' test and true error  \citep[thrm.\ 11.1]{ShaBen14}. Of course, this is still a frequentist bound, so the interpretational issue remains.}

In sum, we have a two-fold theoretical 
justification for ERM, a reliability and a wysiwyg bound---provided that the hypothesis class has the uniform convergence property. But what kind of hypothesis classes satisfy this property?

\subsection{Capacity and the fundamental theorem}\label{ssec:capfundthrm}
The short answer is, \emph{simple} hypothesis classes.

\subsubsection{The VC dimension}
Notions of \emph{capacity} in machine learning are notions of the richness or flexibility of  hypothesis classes, in the sense that a higher-capacity hypothesis class can more easily fit a variety of training data. More precisely, in our framework of binary classification, for a given finite set $X=\{ x_1, \dots, x_m\}$ of instances, a hypothesis class $\mathcal{H}$ has maximal such flexibility-of-fit if, for  each of the $2^m$ possible binary labelings of $X$, there is some $h$ in $\mathcal{H}$ which gives exactly this labeling. In that case we say that $\mathcal{H}$ \emph{shatters}  $X$, and the relevant notion of the capacity of a hypothesis class in our framework is a measure of its ability to shatter sets of instances.

\begin{defn}
The \emph{VC dimension} of hypothesis class $\mathcal{H}$ 
is the maximal size of a set $X \subset \mathcal{X}$ that is shattered by $\mathcal{H}$. If $\mathcal{H}$ shatters sets of arbitarily large size, then the VC dimension of $\mathcal{H}$ is infinite. A \emph{VC class} is a class with finite VC dimension. 
\end{defn}

%

For example, the hypothesis class $\mathcal{H}^\mathrm{pol}_1$ of linear separators has a strictly smaller VC dimension than the class $\mathcal{H}^\mathrm{pol}_2$ of quadratic separators. Both are still VC classes, which is no longer so for the class  $\mathcal{H}^\mathrm{pol}$ of \emph{all} polynomial separators. Another example of a class with infinite VC dimension is the class $\mathcal{H}^\mathrm{sine}=\{ x \mapsto \sin \alpha x \}_{\alpha \in \mathbb{R}}$ of sine functions. No matter how large a given finite set of labeled instances, the corresponding separation in positive and negative classes can be achieved by a sine function with apposite choice of period parameter $\alpha$ \citep[p.\ 82]{Vap00}.

\subsubsection{Simplicity}\label{sssec:simpl}
The latter example shows that this notion of capacity does not need to align with traditional notions of simplicity or parsimony in terms of number of parameters. Such an alignment still holds nicely for the polynomials, where VCdim($\mathcal{H}^\mathrm{pol}_i) <$ VCdim($\mathcal{H}^\mathrm{pol}_j$) for degrees $i<j$; but the class of sine functions above is specified with only one parameter, yet has maximal capacity.

However, as discussed by \citet[sect.\ 3]{Ste25mam}, the capacity notion of VC dimension  gives \emph{a} plausible notion of simplicity. Capacity formalizes a conception of complexity as flexibility-of-fit, where a simpler, more parsimonious class covers fewer possible patterns.\footnote{\citet{DCGGHLMMNRSWWBS25pnas}, 
	in a recent overview of Occam's razor in science, call this \emph{parsimony by constraints}.} Moreover, capacity expresses an ``inherent'' complexity of a class \citep{Rom17pos,Gru07}, because it is invariant to inessential choices of how to describe the class. This is a kind of robustness which definitions of the complexity of individual hypotheses (and definitions of the complexity of classes derived from those) must lack.\footnote{Definitions 
	that seek to capture how ``bumpy'' or ``wiggly'' a single hypothesis looks are inevitably not robust under trivial 1-1 transformations of the coordinate space \citep[p.\ 783]{Kie01bjps}. Similarly \citep{Pri76pos}, parameter counting is not robust under trivial redescriptions which introduce fewer or more parameters (\citealp[sect.\ 7]{Tur90bjps}; \citealp[pp.\ 34f]{Kie97bjps}).}\textsuperscript{,}\footnote{Capacity 
	is in a sense independent of the ``complexity'' of individual hypotheses: a class with a small number of highly ``complex'' hypotheses (say, high-$n$-degree polynomials in your favorite coordinate system) still has small capacity (see again \citealp{Ste25mam}). The best that can be said is that ways of quantitatively defining the simplicity of individual objects (like in terms of description length) will normally be such that there are \emph{more} complex than simple objects, so that the class with all objects at a certain complexity level will be larger(-capacity) with a higher level. This could go some way towards an explanation why capacity often tracks conceptions of the simplicity of individual hypotheses.} 

More precisely, notions of capacity are invariant under redescriptions which leave the structure of the learning problem untouched---which is an immediate corollary of their role in formal learning guarantees.\footnote{For 
	instance, \citet[sect.\ 5]{Ste09jpl} writes that while it ``cannot be altered by faithful translations of the hypotheses it contains, VC dimension does depend on what counts as an
individual data unit,'' and considers redescribing the learning problem by collapsing every pair of observations into a single one. This makes the VC dimensions smaller; but it also changes the problem, as data sizes are now artificially smaller, too. The formal connection (as given by theorem \ref{thrm:funpac} below) between capacity and data sizes for bounds of uniform convergence  is, as it must, preserved.}
	Indeed, the primary motivation for the definition of VC dimension is not an independent intuitive appeal as a simplicity measure (even if, again, it has such an appeal), but a provable connection, indeed \emph{equivalence}, to uniform convergence.


	
\subsubsection{The fundamental theorem}
The central result of \citet{VapChe71tpa}, celebrated as the fundamental theorem of statistical learning theory \citep[thrm.\ 6.7]{ShaBen14}, is that finite VC dimension characterizes uniform convergence. 

\begin{thrm}[The fundamental theorem]\label{thrm:funpac}
An hypothesis class has the uniform convergence property if and only if it is a VC class.
%
\end{thrm}

In particular, ERM$_\mathcal{H}$ satisfies the learnability bound \eqref{eq:ermlearn} precisely if 
hypothesis class $\mathcal{H}$ has finite VC dimension.\footnote{Learnability is indeed equivalent to learnability by ERM: ERM satisfies learnability if any algorithm does \citep[thrm.\ 6.7]{ShaBen14}.} Further, a more fine-grained, quantitative statement of the fundamental theorem tells us that the learnability bound depends on the VC dimension. Namely, we have that, for some constant $c$,
\begin{align}\label{eq:wysiwygerrorquant1}
 \epsilon^\mathrm{uc}_\mathcal{H}(m,\delta) = c B,
\end{align}
where
\begin{align}\label{eq:wysiwygerrorquant2}
B = \sqrt{\frac{\mathrm{VCdim}(\mathcal{H})-\log \delta}{m}}.
\end{align}

\subsection{The argument}\label{ssec:corearg}
The methodological lesson from the fundamental theorem is that \emph{in order to} profit from the wysiwyg and reliability guarantees, one must keep $\mathcal{H}$ simple. This is a means-ends argument for simplicity.

\subsubsection{Means-ends} I borrow this terminology from the approach in the philosophy of science which goes by the name of \emph{formal learning theory} or also \emph{means-ends epistemology} (\citealp{Kel96,Gen18phd,Sch17sep}; also see \citealp[sect.\ 4.2]{Ste25mam}). The driving idea here is that  inductive problems should be analyzed in terms of what types of reliability (ends) are attainable with what assumptions and methods (means); a perspective which aligns well with machine learning theory.\footnote{The approach indeed grew out of the theoretical computer science branch of \emph{algorithmic learning theory} \citep{JaiOshRoySha99}.} 

For instance, we can formulate, for a certain type of learning problem, an interesting notion of reliability, and ask for what kind of assumptions this end is achievable. This question would be answered by a characterization theorem, which lays down the necessary and sufficient conditions for the attainability of this end. In Kelly's words, ``such results may be thought of as \emph{transcendental deductions} for
reliable inductive inference, since they show what sort of knowledge
is necessary if reliable inductive inference is to be possible''  \citeyearpar[p.\ 74]{Kel96}.

The fundamental theorem is such a characterization result in  our setting. It makes precise what sort of knowledge (a low-complexity hypothesis class) is necessary for reliable inductive inference (a learnability guarantee for ERM). This gives a means-ends reason to choose a simple hypothesis class.

\subsubsection{The norm}\label{sssec:normn1} 
The means-ends argument turns the theoretical justification for ERM into a methodological justification for the norm of preferring simplicity. The means-ends argument can be made precise as an argument from two premises.\footnote{Thanks to an anonymous reviewer for this suggestion.} 

\renewcommand{\labelenumi}{\arabic{enumi}.}

\begin{enumerate}
\item In order to have the wysiwyg and reliability guarantees \eqref{eq:ermwysiwyg} and \eqref{eq:ermlearn}, one must keep the hypothesis class simple.
\item One should aim to have the wysiwyg and reliability guarantees \eqref{eq:ermwysiwyg} and \eqref{eq:ermlearn}.
\end{enumerate}

The conclusion of this argument yields the methodological simplicity norm that one should aim to keep the hypothesis class simple.

\begin{quotation}
\textbf{Methodological norm \customlabel{norm:m1}{N1}(Occam's razor).} Keep the hypothesis class simple.
\end{quotation}

For example, \eqref{norm:m1} advises against the use of the class $\mathcal{H}^\mathrm{pol}$ of all polynomials, because it is too complex---too complex to retain the guarantees for ERM. Moreover, while I stated the argument and the norm here in a categorical manner, the quantitative version of the fundamental theorem supports a more continuous reading: in order to have stronger guarantees, one should aim to keep the model simpler. For instance, the norm advises us to prefer the class $\mathcal{H}^\mathrm{pol}_1$ of linear separators to the class $\mathcal{H}^\mathrm{pol}_2$ of quadratic separators, because the former comes with stronger guarantees for ERM.\footnote{A
	slightly different gloss on the theoretical guarantees is that they ``tell us \emph{how rich a
function space} we can afford to search on the `budget' given by our sample size, while maintaining the quality of the estimate of the accuracy of the best‐fitting function'' \citep[p.\ 3]{GroGenSul24pc}. But for given ``budget'' simpler is theoretically still better because the guarantees are stronger.}

\subsubsection{Constraints and limitations}\label{sssec:norm1constrs}
However, if we follow this norm all the way, then we end up choosing a \emph{maximally} simple class, a class of VC dimension 0. That clearly does not make sense in any actual learning problem, since a maximally simple class of VC dimension 0 is a singleton class with only one classifier, meaning that there would be no learning problem left.



This exposes the weak point in the argument, namely the second premise. The extent to which one should or reasonably could aim for the relevant guarantees depends on the context of the actual learning problem, in which further epistemic as well as pragmatic factors are at play. For one thing, practical considerations dictate what error and/or confidence bounds would be acceptable, or what training sample sizes are available to us. For another, an essential epistemic factor is what we know about the domain, which informs what assumptions we would be willing to make, and so what hypothesis classes we would still find reasonable. 


The latter is especially important because the reliability justification for ERM is \emph{model-relative} \citep{SteGru21syn}: it is about locating the best hypothesis in the \emph{inductive model}, the hypothesis class. The best hypothesis in the class might still not be very good; so it is important to also choose a good class, a class which we expect contains good classifiers. In machine learning terminology, we want the inductive model to have a good \emph{inductive bias}.


These considerations act as a check on the theoretical push towards simplicity that is encoded in \eqref{norm:m1}. The epistemic norm \eqref{norm:m1}, underwritten by the theoretical justification for a simple inductive model, pushes us in the direction of simplicity; but how far we can go in this direction, to what extent we can reasonably follow this methodological norm (to what extent we go along with premise 2) depends on the specifics of the actual learning situation. In particular, our knowledge about the domain informs what hypothesis classes could still be expected to contain good classifiers. This knowledge normally acts as a check on or  counterpull to the simplicity norm, because simpler hypothesis classes generally mean stronger assumptions.

There are thus pragmatic and epistemic constraints to the application of simplicity norm \eqref{norm:m1}. But at this point the worry arises that these factors are normally so constraining that \eqref{norm:m1} is hardly applicable at all.

\subsubsection{Beyond the core argument}
In the older days of machine learning, we can see the simplicity norm \eqref{norm:m1} routinely evoked and followed.\footnote{For 
	instance, in their seminal paper introducing convolutional neural nets for handwritten digit recognition, \citet[p.\ 541]{LBDHHHJ89nc} write that ``the basic design principle is to reduce the number of free parameters in the network as much as possible without overly reducing its computational power. Application of this principle increases the probability of correct generalization because it results in a specialized network architecture that has [\dots]\ a reduced Vapnik-Chervonenkis dimensionality.''} 
	But certainly since the deep learning age the field is rather characterized by a lack of strong modeling assumptions. In the words of \citet{BarMonRak21ac}, ``deep learning is a data-driven approach: these are rich but generic models, and the architecture, parametrization and nonlinearities are typically chosen without reference to a specific model for the process generating the data.'' This does not sit well with the simplicity norm \eqref{norm:m1}, because simple (small-capacity) hypothesis classes are much more restrictive, and as such do commit one to strong assumptions. 
	The worry is therefore that simplicity norm \eqref{norm:m1} is no longer relevant to modern machine learning.\footnote{Another
	worry is that this version of Occam's razor is ``a rather different
recommendation than the usual exhortation to select the simplest hypothesis compatible with the data'' \citep[fn.\ 6]{GroGenSul24pc}. The simplicity norm to be discussed in the following, building on the core argument but underwriting regularization in the learning, is a more direct instantiation of the usual Occam norm to prefer simplicity in the inductive inference from data to hypothesis.}
	
Simplicity does still appear to play an important methodological role, however, even in deep learning. It may no longer play a significant role in the initial choice of hypothesis class, but it does pop up again in the subsequent learning process, namely in the standard technique of \emph{regularization}. The theoretical underpinnings for this technique lead us to Vapnik's second ``inductive principle,'' and a methodological simplicity argument which expands on the core argument of this section.


\section{Structural risk minimization}\label{sec:srm}

\subsection{Generalized uniform convergence and SRM} 

The push towards simplicity encoded in \eqref{norm:m1} is checked or countered by, in particular, the assumptions we are willing to make. Simpler inductive models generally also encode stronger assumptions, stronger inductive biases: the best hypothesis in a simpler class might be expected to be worse than the best hypothesis in a more complex class. This tension between simpler classes with stronger inductive bias and more complex classes with weaker inductive bias is expressed in a classical  trade-off.

\subsubsection{The bias-complexity trade-off} 
The true risk of a learned hypothesis $\hat{h}$ can trivially be decomposed as a sum of two errors, the true risk of the best hypothesis in the given hypothesis class $\mathcal{H}$ (the \emph{approximation error}) and the difference between the true risk of $\hat{h}$ and that of the best hypothesis  $\mathcal{H}$ (the \emph{estimation error}):
\begin{align}
L_\mathcal{D}(\hat{h}) = \underbrace{\min_{h \in \mathcal{H}} L_\mathcal{D}(h)}_{\textrm{appr.\ error}} + \underbrace{L_\mathcal{D}(\hat{h}) -\min_{h \in \mathcal{H}} L_\mathcal{D}(h)}_{\textrm{est.\ error}}.
\end{align}

The interplay between these two terms is depicted in figure \ref{fig:biasvar}. The simplicity norm \eqref{norm:m1} is concerned with minimizing the estimation error, or preventing \emph{overfitting}: the selection of a hypothesis which is significantly worse than the best in the class. It pushes us to the left along the x-axis, to simpler hypothesis classes, with (provably, with high probability) smaller difference between the two curves. However, at a more informal level, less complex classes might be expected to contain fewer good hypotheses, so that the approximation error is higher. To prevent \emph{underfitting}, when even the best hypothesis in the class is not good, we would need to move to the right, towards more complex classes. 

\begin{figure}
\begin{tikzpicture}
  \begin{axis}[
      width=10cm,
      height=6cm,
      axis lines=middle,
      xlabel={capacity of $\mathcal{H}$},
      xlabel style={yshift=-1.8em},
      xtick=\empty,                
      ytick=\empty, 
      xmin=0, xmax=10,
      ymin=0, ymax=6,
      domain=0:10,
      tick label style={font=\tiny,opacity=0},
    ]

    \addplot[thick] {0.15*(x-5)^2 + 1.7}
      node[pos=0.35,anchor=south,yshift=3mm, 
             ] {\small true risk};

    \addplot[thick, dashed] {5.25 * 1.35^(-x - 0.1 * x^2)+0.2}
    node[pos=0.60,anchor=south,yshift=2mm,xshift=1mm,
             ] {\small appr.\ error};;

 \pgfmathsetmacro{\yVar}{0.4}          
    \pgfmathsetmacro{\yTot}{3.05} 

    \draw[densely dashed,very thick] 
        (axis cs:8,\yVar) -- (axis cs:8,\yTot);

    \node[rotate=90,above] at (axis cs:8,{(\yVar+\yTot)/2})
          {\small est.\ error};

  \end{axis}
\end{tikzpicture}
\caption{The bias-complexity trade-off.}\label{fig:biasvar}
\end{figure}
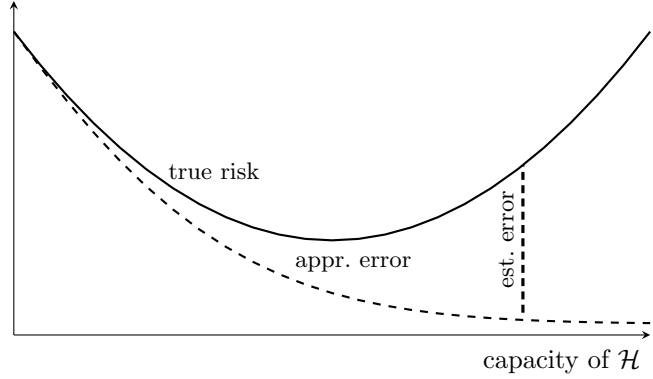

There is an asymmetry here, in that the theory, and the resulting epistemic norm \eqref{norm:m1}, covers only one side, namely the minimization of estimation error. In practice, we would want to find the sweet spot in the figure, where the selected hypothesis does not just have low estimation error, but low total error, low true risk. This asks for a more informal assessment what  class, given what we know about the specific learning problem, is likely to have good hypotheses, to have low approximation error. Such an assessment constrains the application of norm \eqref{norm:m1}.\footnote{The asymmetry in the theoretical focus on estimation error could thus be seen as problematic for the reliability justification for a simple class: the model-relative reliability of ERM counts for little if the approximation error is high. This, however, leaves untouched the wysiwyg justification: it will always be helpful to get an indication of \emph{how} good or bad the (best in the) class is. Also see \citep[p.\ 34]{Gru07} on ``the inherent difference between under- and overfitting.''}

Purely theoretical analysis cannot tell us how to strike the right balance in the choice of hypothesis class prior to the learning. However, it turns out that, to some extent, the theory can tell us how to strike the balance \emph{in} the learning. The key is a generalization of the uniform convergence theorem, that moves the analysis to the level of \emph{multiple} hypothesis classes.


\subsubsection{Generalized uniform convergence} 
Instead of a single hypothesis class somewhere on the complexity-axis of figure \ref{fig:biasvar}, imagine we pick a countable \emph{sequence} of different hypothesis classes along the axis. Formally, let $(\mathcal{H}_n)_{n \in \mathbb{N}}$ be such a sequence, where each $\mathcal{H}_n$ is a VC class. I will refer to the union $\mathcal{H} = \cup_{n \in \mathbb{N}} \mathcal{H}_n$ as the \emph{super}class, and to the individual $\mathcal{H}_n$ as the \emph{sub}classes.

By the fundamental theorem, each subclass $\mathcal{H}_n$  has the uniform convergence property, with a corresponding accuracy function $\epsilon_n^\mathrm{uc}$. Additionally, we define a \emph{weight function} $w: \mathbb{N} \rightarrow [0,1]$ with $\sum_n w(n) \leq 1$, assigning each subclass $\mathcal{H}_n$ a numerical weight $w(n)$ in the unit interval.

For instance, we could choose the sequence $(\mathcal{H}^\mathrm{pol}_n)_{n \in \mathbb{N}}$ of all the subclasses of polynomial separators of degree $n$. Each  $\mathcal{H}^\mathrm{pol}_n$ is a VC class; and the superclass $\cup_{n \in \mathbb{N}}  \mathcal{H}^\mathrm{pol}_n$ is the class $\mathcal{H}^\mathrm{pol}$ of all polynomials. For the weights, we could, for instance, pick the function $w: n \mapsto 2^{-n}$, which satisfies the property that the sum of all weights does not exceed 1.

Now one can show the following.\footnote{This 
	is theorem 7.4 in \citep{ShaBen14}. 
	}

\begin{thrm}[Generalized uniform convergence]\label{thrm:genuniconv}
Given hypothesis subclass sequence $(\mathcal{H}_n)_{n \in \mathbb{N}}$ such that each $\mathcal{H}_n$ has the uniform convergence property with accuracy function $\epsilon_n^\mathrm{uc}$, and weight function $w$. Then for all $\delta \in (0,1)$ and all $\mathcal{D}$ we have with probability at least $1-\delta$
\begin{align}\label{eq:wysiwyggen1}
(\forall n \in \mathbb{N})(\forall h \in \mathcal{H}_n) \left[ | L_\mathcal{D}(h) - L_S(h) | \leq \epsilon_n^\mathrm{uc}(m,w(n)\delta) \right],
\end{align}
and so in particular, for $\mathcal{H}=\cup_{n \in \mathbb{N}}\mathcal{H}_n$,
\begin{align}\label{eq:wysiwyggen2}
(\forall h \in \mathcal{H}) \left[ L_\mathcal{D}(h) \leq L_S(h) + \min_{n: h \in \mathcal{H}_n} \epsilon_{n}^\mathrm{uc}(m,w(n)\delta) \right].
\end{align}
\end{thrm}

Like in the original uniform convergence result, this wysiwyg bound uniformly holds for all $h \in \mathcal{H}$ simultaneously, but in this case with an accuracy that does depend on the hypothesis subclass(es) that $h$ is in. It depends on $h$, first, because of the accuracy function $\epsilon^\mathrm{uc}_n$ associated with a subclass $\mathcal{H}_n$ that contains $h$. Second, it depends on $h$ because of the factor $w(n)$ that is applied to the confidence parameter $\delta$ in the accuracy function.

\subsubsection{Structural risk minimization} 
Analogously to how we defined (sect.\ \ref{sssec:ermdef}) the ERM rule as minimizing the uniform convergence wysiwyg bound \eqref{eq:wysiwyg3}, we define the SRM rule \citep{VapChe74} as minimizing  wysiwyg bound \eqref{eq:wysiwyggen2}.

\begin{defn}[SRM]
Given hypothesis class sequence $(\mathcal{H}_n)_{n \in \mathbb{N}}$ (such that each $\mathcal{H}_n$ has the uniform convergence property with accuracy function $\epsilon_n^\textrm{uc}$) and weight function $w$. Let $\mathcal{H}=\cup_{n \in \mathbb{N}}\mathcal{H}_n$. The SRM rule for $(\mathcal{H}_n)_n$ and $w$ is defined by
\begin{align}\label{eq:minsrm}
\mathrm{SRM}_{(\mathcal{H}_n)_n}^w(S) \in \argmin_{h \in \mathcal{H}} \left[ L_S(h) + \min_{n: h \in \mathcal{H}_n} \epsilon_n^\textrm{uc}(m,w(n)\delta) \right].
\end{align}
\end{defn}

\subsection{Regularization}\label{ssec:srmreg}
The SRM rule minimizes, not just the empirical error, but the sum of the empirical error and an additional term. This additional term depends on the hypothesis $h$, or more precisely, on the earliest hypothesis subclass $\mathcal{H}_n$ in the sequence that the hypothesis is in. 

\subsubsection{The regularization term}
SRM can thus be seen to strike a balance between empirical error and an additional \emph{penalty} or (in machine learning terminology) \emph{regularization} term for $\mathcal{H}_n$. This regularization term
\begin{align}\label{eq:regterm}
\min_{n: h \in \mathcal{H}_n} \epsilon_n^\textrm{uc}(m,w(n)\delta)
\end{align}
is, by \eqref{eq:wysiwygerrorquant1} and \eqref{eq:wysiwygerrorquant2}, up to a multiplicative constant equal to\footnote{In 
	order to actually implement the SRM rule \eqref{eq:minsrm}, we would of course need to settle on exact values for the penalty term. This is further complicated by the fact that we often only have loose upper bounds for the VC dimension \citep[p.\ 239]{HasTibFri09}. Similarly to the case of ERM, I here abstract away from implementation issues.}
	\begin{align}\label{eq:regtermexpl}
\sqrt{\frac{\mathrm{VCdim}(\mathcal{H}_n)-\log w(n) - \log \delta}{m}}.
\end{align}

The regularization term is larger with higher VC dimension of $\mathcal{H}_n$ and with lower weight $w(n)$. If we interpret the regularization as a penalty for $\mathcal{H}_n$'s \emph{complexity}, then SRM can be seen to strike a balance between empirical error and simplicity. The rule automatically navigates the bias-complexity trade-off by balancing empirical error (as a proxy for approximation error) with simplicity (as a proxy for estimation error), finding ``the best trade-off between the approximation error and a distribution-free upper bound on the estimation error'' \citep{BarBouLug02ml}.

\subsubsection{Model selection}
In the two-step picture of model selection (sect.\ \ref{sssec:curvefit} above), SRM can thus be seen to automatize both steps: it selects (step 1) a particular ``model'' (subclass in the sequence) and (step 2) the lowest-empirical-error hypothesis in this model (subclass).\footnote{Of
	 course, it is still up to us to choose the subclass sequence (sect.\ \ref{ssec:srmmodelrel} below).} 			This makes SRM structurally similar to  methods like AIC and BIC, which likewise select a model by minimizing penalized fit.\footnote{\label{fn:holdout}We 
	do not have to automatize model selection to this extent. An alternative is to, first, run ERM$_{\mathcal{H}_n}$ for each of the subclasses $\mathcal{H}_n$ (realistically, a finite number $N$) on the same training data set, resulting in a set $\mathcal{H}_\mathrm{val}$ of $N$ selected hypotheses. Second, we run ERM$_{\mathcal{H}_\mathrm{val}}$ on a separate \emph{validation} data set to select a final hypothesis. The latter step is to catch overfitting, and so the whole procedure is again meant to manage the bias-complexity trade-off. From this perspective, SRM can be seen to ``approximate the validation step [\dots] automatically'' (\citealp[p.\ 223]{HasTibFri09}; also see \citealp[ch.\ 11]{ShaBen14}). In practice, especially if data is plenty, the more manual validation procedure 
	might be preferable to implementing regularization in the learning algorithm, like in SRM.   }  

\subsubsection{The  norm}
If we interpret the regularization penalty as a complexity penalty, then the SRM rule implements a methodological simplicity norm.
\begin{quotation}
\textbf{Methodological norm \customlabel{norm:m2}{N2}(Occam's razor).} Trade fit for simplicity.
\end{quotation}

But what exactly is the justification for this norm? Why indeed is the specific trade-off that SRM implements ``the right'' \citep{ShaBen14} or ``the best''  \citep{BarBouLug02ml}, and is the regularization term really best understood as a complexity penalty? In order to address these questions satisfactorily, we have to take a closer look at the theoretical justification for the SRM rule.



\section{The theoretical justification for SRM}\label{sec:srmjust}
I first discuss the best theoretical justification for SRM (section \ref{ssec:srmthjusts}). I then highlight an apparent obstacle to a subsequent argument for simplicity norm \eqref{norm:m2}, namely the model-relativity of this theoretical justification (section \ref{ssec:srmmodelrel}).

\subsection{Theoretical justifications}\label{ssec:srmthjusts}
Since the worry about simplicity norm \eqref{norm:m1} and its theoretical justification was its one-sided concern with estimation error, a good starting point is to see whether in the case of SRM more can be said about the approximation error.

\subsubsection{Universal consistency}\label{sssec:just:univcons}
The reliability guarantee of PAC learnability  (sect.\ \ref{sssec:pac}) entails that, as we draw more and more training data, the estimation errors of the selected classifiers converge (in probability) to 0.  What if we aim for the \emph{true risks} (so  estimation \emph{and} approximation errors) to converge (in probability) to 0?

As discussed by \citet[sect.\ 7]{LuxSch11incl}, this can be achieved in two steps.\footnote{\citet{LuxSch11incl} do not mention SRM, but as explained this rule is a way of implementing what they describe. Also see \citep[ch.\ 18]{DevGyoLug96}, which explicitly deals with SRM and which von Luxburg and Sch\"olkopf base part of their discussion on.} First, we choose a sequence of (nested) subclasses $(\mathcal{H}_n)_{n \in \mathbb{N}}$ such that each possible hypothesis is eventually contained in (or at least arbitrarily closely approximated by) some $\mathcal{H}_n$. 
Second, we have to devise a learning procedure that only has access to an initial subsequence of subclasses which grows at the right pace with the sample size, namely in such a way that both estimation and approximation error converge to 0. The SRM rule gives such a method that is \emph{Bayes-consistent} \citep[thrm.\ 18.1]{DevGyoLug96}.

Formally, following again \citet{ShaBen14}, we can understand consistency as a learnability notion strictly weaker than PAC learnability, because the sample complexity function depends on more elements. We say that algorithm $A$ is \emph{universally consistent} w.r.t.\ hypothesis class $\mathcal{H}$ 
if there is a sample complexity function $m^\mathrm{con}_\mathcal{H}: (0,1)^2 \times \mathcal{H} \times \mathcal{P} \rightarrow \mathbb{N}$ such that for every $\epsilon,\delta \in (0,1)$, every $h \in \mathcal{H}$ and every $\mathcal{D}$, 
we have for $m \geq 	m^\mathrm{con}_\mathcal{H}(\epsilon,\delta,h,\mathcal{D})$ that 
\begin{align}\label{eq:consbound}
\mathbb{P}_{S \sim \mathcal{D}^m}\left[ L_\mathcal{D}(A(S)) \leq  L_\mathcal{D}(h) + \epsilon \right] \geq 1-\delta.
\end{align}	
We further say 
(now following again \citealp[p.\ 660]{LuxSch11incl}) that $A$ is \emph{universally Bayes-consistent} if $\mathcal{H}$ is the class of \emph{all} hypotheses.

This might look like an impressive property, but \citet{ShaBen14} disagree. The reason is that we have no control over the speed of convergence. This can be seen from the dependence of the sample complexity function on (aside from a reference $h$) the true distribution, which we do not normally know (cf.\ \citealp[sect.\ 7.4]{LuxSch11incl}). Indeed, on a closer look, consistency is so weak that it is satisfied by learning algorithms which do not look reasonable at all. \citet[pp.\ 15f]{ShaBen14} introduce a ``bad learner'' which simply memorizes the training data and gives constant prediction 0 for unseen instances, but which is universally Bayes-consistent for a countable domain $\mathcal{X}$:\footnote{The
	countability of the domain is important here: \citet[sect.\ 4.2]{LuxSch11incl} give the same  \texttt{Memorize} algorithm (note: an instantiation of the ERM rule) for the domain $[0,1]$ as an example of how ERM can \emph{fail} to be consistent.} 
\begin{quote}
Intuitively, it is not obvious that the \texttt{Memorize} algorithm should be viewed as a \emph{learner}, since it lacks the aspect of generalization [\dots]. The fact that \texttt{Memorize} is a consistent algorithm [\dots]\ therefore raises doubt about the usefulness of consistency guarantees. (ibid., p.\ 67)
\end{quote}

But we do not here need to settle the question whether consistency can still count as a minimal kind of justification for learning algorithms,\footnote{Consistency
	could serve as a ``sanity check'' (cf.\ \citealp[sect.\ 17.1.1]{Gru07}), or a necessary but not sufficient condition for justification. This may be consistent with Shalev-Shwartz and Ben-David's view that ``[s]ince it is easy to make any algorithm consistent, it may not be wise to prefer one algorithm over the other just because of consistency considerations'' \citeyearpar[p.\ 69]{ShaBen14}.} 
because we can say something stronger about SRM.\footnote{In discussing the reliability of SRM, \citet[sects.\ 3.3--3.4]{HarKul07} only mention the property of universal consistency. This is criticized by \citet{KelMay08ndprinc}, who also highlight the kind of justification I come to in section \ref{sssec:just:srmwysiwyg} below. }

\subsubsection{Nonuniform learnability}\label{sssec:just:nonunif} 
\citet[ch.\ 7]{ShaBen14} introduce SRM in the context of a notion of learnability strictly between PAC learnability and consistency. The notion of \emph{nonuniform learnability} involves a sample complexity function which (unlike consistency) does not depend on the unknown true distribution, but which (unlike PAC learnability) does depend on a reference hypothesis $h$. Thus we say that $A$ nonuniformly learns hypothesis class $\mathcal{H}$ if there is a sample complexity function $m^\mathrm{nul}_\mathcal{H}: (0,1)^2 \times \mathcal{H} \rightarrow \mathbb{N}$ such that for every $\epsilon,\delta \in (0,1)$ and every $h \in \mathcal{H}$, 
we have for $m \geq 	m^\mathrm{nul}_\mathcal{H}(\epsilon,\delta,h)$ 	the property \eqref{eq:consbound}.\footnote{The notion of nonuniform learnability appears to have been introduced by \citet{BenIta88icalp,BenIta94jcss}. \citet{BEHW89jacm,LinManRiv91ic} discuss the same notion in the PAC setting.}

\citet[thrm.\ 7.2]{ShaBen14} state a generalization of the fundamental theorem, which characterizes nonuniform learnability. Namely, a hypothesis class $\mathcal{H}$ is nonuniformly learnable if and only if it is a countable union $\cup_n \mathcal{H}_n$ of VC classes, and indeed if and only if it is nonuniformly learnable by SRM.\footnote{It 
	is relatively straightforward to show that any nonuniformly learnable class can be decomposed into a countable union  of individually PAC learnable classes (\citealp[p.\ 60]{ShaBen14}; where they 		refer in the final step to the fundamental theorem, what is more precisely required is their 			corollary 6.4). The other direction follows by showing that a countable union of VC classes is 		nonuniformly learnable by SRM (ibid., thrms.\ 7.3, 7.4).}\textsuperscript{,}\footnote{	This  
	implies that nonuniform learnability is a strict relaxation of PAC learnability (there are nonuniformly learnable but not PAC learnable classes, namely any countable union of VC classes which does not itself have finite VC dimension; ibid., exmpl.\ 7.1), and also that it is a strict strengthening of consistency (we saw that ERM is universally consistent on countable domains, but on infinite domains the class of all hypotheses is not nonuniformly learnable; sect.\ \ref{sssec:nfl} below).} 
	Further, theorem \ref{thrm:ermfailnonunif} in the appendix shows that if the superclass does not  already have finite VC dimension, one really does need a different method than ERM for the nonuniform learnability guarantee: while, as we saw, ERM can be consistent on non-VC hypotheses classes, it cannot non-uniformly learn any such class.
	
%

However, it is again not clear whether nonuniform learnability is actually a useful notion of learnability. 
 In the words of Shalev-Shwartz and Ben-David themselves,
\begin{quote}
When approaching a learning problem, a natural question is how many [data instances] we need to collect in order to learn it. Here, PAC learning gives a crisp answer. However, for both nonuniform learning and consistency, we do not know in advance how many [instances] are required to learn $\mathcal{H}$. In nonuniform learning this number depends on the best hypothesis in $\mathcal{H}$, and in consistency it also depends on the underlying distribution. In this sense, PAC learning is the only useful definition of learnability. \citeyearpar[p.\ 67]{ShaBen14}
\end{quote}

Having such a bound on sample complexity, though, is only one of three possible uses  of the theoretical analysis which the authors put forward (ibid., sect.\ 7.5). The first  is to have ``an upper bound on the true risk of the learned hypothesis.''\footnote{The third in their list is to have ``a crisp way to encode prior knowledge'' \citep[p.\ 68]{ShaBen14}. Here they also discuss SRM for model selection, writing that ``the SRM rule enables us to select the right model on the basis of the data itself'' (ibid.). But, again, it is not  fully clear why SRM does this in the ``right'' way; and I suggest that the answer lies in the bounds to be discussed next. I will turn to the role of prior knowledge in SRM in section \ref{ssec:srmmodelrel}.} This points at a kind of justification that is closest to the original derivation of SRM.


\subsubsection{Wysiwyg and oracle bounds}\label{sssec:just:srmwysiwyg}
Here we switch again to the perspective where we draw a training sample of a specific size $m$, and we are interested in a bound on the true risk of the hypothesis that will be selected. We now reason in a way that is analogous to the justification for ERM (sect.\ \ref{sssec:ermjust1}), albeit with some differences.\footnote{Note
	that having these bounds is not a stronger property than nonuniform learnability in the way that nonuniform learnability is a stronger property than consistency. This is analogous to the wysiwyg and reliability bounds for ERM, which is not a stronger property than learnability; for both it is necessary and sufficient to have a VC hypothesis class. In the current case the necessary and sufficient condition for both is  a countable sequence of VC classes. The value of these finite-sample bounds is rather that they are more clearly methodologically useful, or so I will argue in more detail in the upcoming sections.}

First, from the generalized uniform convergence bound \eqref{eq:wysiwyggen1} or \eqref{eq:wysiwyggen2} for hypothesis class sequence $(\mathcal{H}_n)_n$ and weight function $w$, applied to the particular hypothesis selected by SRM$_{(\mathcal{H}_n)_n}^w$, we have (for size-$m$ sample generated from any distribution, with probability at least $1-\delta$)
\begin{align}\label{eq:srmwysiwyg}
 L_\mathcal{D}(\mathrm{SRM}_{(\mathcal{H}_n)_n}^w(S)) \leq L_S(\mathrm{SRM}_{(\mathcal{H}_n)_n}^w(S)) +  \epsilon_{\hat{n}}^\mathrm{uc}(m,w(\hat{n})\delta).
\end{align}

Here $\hat{n}$ is the index of the hypothesis class $\mathcal{H}_{\hat{n}}$ that SRM selected the hypothesis from. This is again a wysiwyg bound, with the difference that the error term now depends on the hypothesis (or rather, on the class $\mathcal{H}_{\hat{n}}$ it was selected from).  

Furthermore, by the definition of SRM as a minimizer and another application of the generalized uniform convergence bound \eqref{eq:wysiwyggen1},
\begin{align}
L_S(\mathrm{SRM}_{(\mathcal{H}_n)_n}^w(S)) + \epsilon_{\hat{n}}^\mathrm{uc}(m,w(\hat{n})\delta) &\leq \min_{n, h \in \mathcal{H}_n} \left[ L_S(h) +  \epsilon_{n}^\mathrm{uc}(m,w(n)\delta) \right] \\
&\leq \min_{n, h \in \mathcal{H}_n} \left[ L_\mathcal{D}(h) + 2\epsilon_{n}^\mathrm{uc}(m,w(n)\delta)\right] \\
&\leq \min_{h \in \mathcal{H}}  L_\mathcal{D}(h) + 2\epsilon_{n_\mathcal{D}^*}^\mathrm{uc}(m,w(n_\mathcal{D}^*)\delta).
\end{align}

Here $n_\mathcal{D}^*$ is the index of the earliest class containing the best (lowest true risk) hypothesis in the superclass, which depends on the true $\mathcal{D}$. In sum, we have
\begin{align}\label{eq:srmoracle}
 L_\mathcal{D}(\mathrm{SRM}_{(\mathcal{H}_n)_n}^w(S)) \leq \min_{h \in \mathcal{H}} L_\mathcal{D}(h) + 2\epsilon_{n_\mathcal{D}^*}^\mathrm{uc}(m,w(n_\mathcal{D}^*)\delta).
\end{align}

This is again a reliability bound. It gives a justification for using SRM, as selecting (with high probability, up to an error term)  the best in the superclass. However, the difference is that the error term now depends on the subclass which contains the best hypothesis. Such a bound is also called an \emph{oracle bound}, because we do not actually know what this class and hence this error term is.\footnote{More generally, we can derive such a bound for any reference subclass (``oracle''), without knowing how good this class actually is. See \citep[exrcs.\ 7.4]{ShaBen14}.}

This reliability bound, not mere consistency, is supposed to show the ``real strength'' of SRM \citep[p.\ 294]{DevGyoLug96}, and why the implemented trade-off is the right or ``optimal'' one (ibid., p.\ 295). However, this bound less clearly constitutes an interesting reliability property than the analogous bound for ERM, because of this further ``oracle'' dependence on what the best class turns out to be. In order to see how exactly this bound can underwrite a methodological justification for SRM, we need to take a step back again, and confront 
the general issue of the relativity of theoretical justification to the prior choice of inductive model.




\subsection{Model-relativity}\label{ssec:srmmodelrel}
The ERM rule is a generic learning rule, which on each application asks for a particular inductive model, a particular VC  class. The fundamental theorem gives a theoretical justification for the ERM rule, but, again (sect.\ \ref{sssec:norm1constrs} above), this justification is model-relative. And the inductive model must be restrictive: it must be a hypothesis class that is sufficiently simple.

Something similar holds for SRM, be it for a more general inductive model and a weaker theoretical guarantee. As \citet[p.\ 68]{ShaBen14} write, we now ``encode our prior knowledge by specifying weights over (subsets of) hypotheses of $\mathcal{H}$'' and then ``we again have a generic learning rule -- SRM.'' The theoretical justification for SRM is again model-relative, relative to this more general, but still restrictive type of inductive model. 

\subsubsection{No free lunch}\label{sssec:nfl}
The no-free-lunch theorem of \citet[thrm.\ 5.1]{ShaBen14} shows that (for an infinite domain $\mathcal{X}$) the class of \emph{all} hypotheses is not PAC learnable (ibid., cor.\ 5.2).\footnote{See \citep{SteGru21syn} for further discussion of the no-free-lunch theorems.} This is a consequence of the fundamental theorem, since the class of all hypotheses does not have finite VC dimension.\footnote{In the presentation of \citet{ShaBen14}, their no-free-lunch theorem forms part of the proof of the fundamental theorem.} 
In the case of nonuniform learning, we have an analogous no-free-lunch result \citep[p.\ 63]{ShaBen14}. Namely, for infinite domain, the class of all hypotheses is not nonuniformly learnable. This is a consequence of the fact that the class of all hypotheses is not a countable union of VC classes (ibid., exrcs.\ 7.5).

The upshot is that there is no ``universal learner'' on either of the two definitions of learnability.\footnote{There 
	\emph{are} universal learners in the sense of universal consistency: an example is the $k$-nearest neighbor classifier \citep[sect.\ 3]{LuxSch11incl}. However, the rate of convergence can be arbitrarily slow: performance up to any finite sample size can be arbitrarily bad (ibid., p.\ 695; \citealp[ch.\ 7]{DevGyoLug96}). This again illustrates the weakness of consistency guarantees.} 
	Both ERM and SRM must be equipped with a restrictive inductive model, which represents a restrictive inductive bias; and the respective learnability guarantees are relative to this choice. In the case of ERM, the inductive model must be a VC class, and learnability gives a model-relative justification of probably finding the near-best in the class. In the case of SRM, both the inductive model and the kind of justification are somewhat more intricate. The inductive model is a (weighted) countable sequence of VC classes, which can be seen to involve both a choice of superclass $\mathcal{H}$ and a choice of how to carve up this class in (and assign weights to) subclasses $\mathcal{H}_n$. The oracle bound guarantee is not just relative to the superclass, but also to this choice of (weighted) subclasses.

Moreover, a choice of (weighted) sequence  is automatically a choice of simplicity ordering, because it determines the VC dimensions (and weights) that appear in the regularization term.

\subsubsection{The subclass sequence}\label{sssec:modelrelclasses}
The model-relativity to the choice of hypothesis class sequence is the topic of a recent paper by \citet{BarCevGne22mam}. The authors compare, for two disjoint subclasses $\mathcal{H}_S$ and $\mathcal{H}_C$ (one a lower-VC dimension ``simple'' class, the other a higher-VC dimension ``complex'' class), the ERM rule on the union $\mathcal{H}_S \cup \mathcal{H}_C$ (which is still a VC class) to the SRM rule on the sequence $(\mathcal{H}_S, \mathcal{H}_C)$ with uniform weights. They then prove upper bounds on the sample size for a probabilistic guarantee of selecting a hypothesis from the ``correct'' class, where correctness derives from one of two possible scenarios: either the data is actually sampled from an element plus noise in the first class (a ``simple world'') or from an element plus noise in the second (``complex world''). This bound is sharper for SRM in the simple world than for ERM, but sharper for ERM than for SRM in the complex world.

	
While it may be an overinterpretation that SRM in the latter case ``provably slows down, instead of favoring, the supervised learning process'' \citep[p.\ 23]{BarCevGne22mam},\footnote{Worst-case 
	upper bounds do not yet determine relative convergence behavior in any particular instance. The derivation of lower bounds would be complicated by the fact that the authors' assumptions of either a ``simple'' or ``complex world'' entail corresponding restrictions on the possible true distributions.} 
	the general lesson is correct that theoretical bounds depend on whether and how we carve up the hypothesis superclass, and this choice may be a better or worse match with the actual learning situation. As a simple illustration inspired by Bargagli Stoffi et al., consider an instance space of two features, and the superclass $\mathcal{H} =\mathcal{H}^\mathrm{pol}_2$ of quadratic separators. 

This class has VC dimension 5, so the error term for ERM$_\mathcal{H}$ in the bounds \eqref{eq:ermwysiwyg} and \eqref{eq:ermlearn} is of the order 
\begin{align}\label{eq:ermvsrm1}
\sqrt{\frac{5 -\log \delta}{m}}.
\end{align} 

But we could also choose to use \textsc{SRM} on a uniformly weighted nested decomposition of $\mathcal{H}$ into $\mathcal{H}_1$ of linear separators and $\mathcal{H}_2 = \mathcal{H}$ of quadratic separators. Since VCdim$(\mathcal{H}_1)=4$, the error terms in the bounds \eqref{eq:srmwysiwyg} and \eqref{eq:srmoracle} for SRM$_{(\mathcal{H}_1,\mathcal{H}_2)}^{(.5,.5)}$ are now of the order
\begin{align}\label{eq:ermvsrm2}
\sqrt{\frac{4 -\log \delta+\log 2}{m}} \textrm{ \ and \ } \sqrt{\frac{5 - \log \delta+\log 2}{m}}
\end{align} 
for $\mathcal{H}_1$ and $\mathcal{H}_2$, respectively. 

This means that we have a stronger wysiwyg guarantee for SRM than for ERM for those cases in which the former selects a hypothesis from $\mathcal{H}_1$; but a weaker one for those cases in which SRM selects a hypothesis from $\mathcal{H} \setminus \mathcal{H}_1$. Moreover, we have a stronger reliability guarantee for SRM than the guarantee for ERM in case $\mathcal{H}_1$ contains the best hypothesis in $\mathcal{H}$; but we have a weaker reliability guarantee if it does not. What cases apply, however, is presumed to be unknown in the learning problem: this depends on the true distribution. Thus, depending on the unknown true distribution, it may or may not in fact  be optimal (in the sense of having the strongest above bounds) to implement a (simplicity) preference for $\mathcal{H}_1$ via SRM. 

\subsubsection{The subfamily problem}
The theory obviously cannot tell us what is the best choice in the previous scenario. Similarly, once we have chosen a non-VC superclass (like the class of all polynomials), the theory cannot tell us what is the best way of carving it up in a sequence of subclasses. In practice, we might have intuitions about the most natural way of doing this (like the carving up of the polynomials by degree), and the theoretical guarantees may even provide the basis for an \emph{explanation} why SRM with this standard choice generally works well. But the theory cannot give a \emph{justification} for this particular choice, because the theoretical guarantees are model-relative; and any other choice of carving up translates to corresponding such guarantees (cf.\ \citealp[pp.\ 11f]{Ste25mam}; \citealp[p.\ 358]{For95bjps}).

The analogous point for AIC is made by \citet[p.\ 40]{Kie97bjps}, in the context of a related problem which has received some attention in the philosophical literature. Translated to the current framework, this ``subfamily problem,'' raised by \citet[sect.\ 6; also see \citealp{Kuk95bjps,For95bjps}]{ForSob94bjps}, 
 is why the functioning of SRM, interpreted as a methodological advice  how to trade off simplicity with fit, is not empty. Namely, \emph{after} we have seen the data $S$, we could specify any (say, high-degree polynomial) hypothesis $\hat{h}_S$ with perfect fit on this data, and imagine a hypothesis class sequence which includes the singleton class $\mathcal{H}_i=\{ \hat{h}_S \}$ with only this hypothesis. Now the SRM rule with this inductive model would have selected this hypothesis, and without any trade-offs. How does this rhyme with the methodological advice to trade off fit with simplicity, let alone with any more specific advice such as to prefer, for equal fit, lower to higher-degree polynomials?

To think there is a problem here reveals a misunderstanding of frequentist guarantees. Or more charitably, it brings out that frequentist guarantees can be quite counter-intuitive, or even of restricted use (cf.\ \citealp[p.\ 14]{Ste25mam}). Frequentist guarantees are \emph{pre-sampling} bounds, relative to a \emph{pre-sampling} choice of inductive model, on the probabilities of outcomes. To try to use such bounds in reasoning about a specific outcome, as in the previous, easily runs into confusions. We could, of course, define as our inductive model a hypothesis class sequence which includes the singleton $\{ \hat{h}_S \}$ fitting perfectly the specific data $S$; and this indeed \emph{would have} led the corresponding SRM procedure to select this hypothesis, in the \emph{hindsight} of this specific outcome $S$. But \emph{pre-sampling} there is no probabilistic guarantee of this specific outcome (unless we explicitly assume so in the true distribution), so no theoretical guarantee of perfect fit that might be used to justify the selection of this particular hypothesis. 

The general point about inferring methodological advice directly from the functioning of SRM is the following. If such advice is inferred from pre-sampling and model-relative guarantees, then it must be  pre-sampling and model-relative advice, too: given your choice of inductive model, it is, with high pre-sampling probability, a good idea to trade off fit with a penalty as given by the inductive model. That also means that we cannot infer from these guarantees any ``absolute'' methodological advice to the effect that, say, one should always for equal fit prefer linear to quadratic hypotheses, because whether SRM implements such a preference depends on the choice of inductive model, and the theory does not tell us what is the right such choice. This holds for the choice of superclass and subclasses, which define the VC dimensions in the penalty terms, but also for the weights of the subclasses.


\subsubsection{The weight function}\label{sssec:modelrelweights}
\citet[sect.\ 7.3]{ShaBen14} describe a version of the \textsc{SRM} rule where \emph{only} the weights play a role. In their ``Minimum Description Length'' (MDL) rule, the given countable hypothesis class $\mathcal{H}$ is decomposed in all its singletons, $\mathcal{H} = \cup_n \{ h_n \}$. Since each singleton class has the same VC dimension 1, the VCdim term can simply be crossed out from the regularization \eqref{eq:regtermexpl}, and all that matters are the weights.\footnote{This
	rule is only loosely related to (model selection) methods treated in standard references on the MDL principe \citep{Gru07,GruRoo20ijmi}, and rather an instance of the  PAC-Bayes approach as presented in (\citealp[sect.\ 17.10.2]{Gru07}; \citealp[sect.\ 6.2]{LuxSch11incl}). In other presentations, the PAC-Bayes label primarily means randomnized prediction via the prior/posterior or weight function, which is updated in accordance with minimizing corresponding bounds (\citealp[ch.\ 31]{ShaBen14}; \citealp{Alq24ftml}).}

The weight function now effectively assigns weights to the single hypotheses directly, and the  ``\textsc{MDL}'' label stems from the  suggestion \citep[p.\ 64]{ShaBen14} to assign hypothesis $h$ a weight that depends on $h$'s \emph{description length} via some description language.\footnote{A 
	weight function corresponds to a prefix-free description language, where higher weights correspond to shorter descriptions \citep[sect.\ 7.3]{ShaBen14}.}
	They write that ``the MDL paradigm gives a formal justification to a philosophical principle of induction called Occam's razor'' (ibid., p.\ 58), because 
	``the more complex a hypothesis $h$ is (in the sense of having a longer description), the larger the sample size it has to fit to guarantee that it has small true risk'' (ibid., p.\ 65). 
	
This reasoning falters on the fact that, due to the dependence on choice of description language, description length is not a robust notion of the simplicity of individual hypotheses;\footnote{Recall section \ref{sssec:simpl}, and also see \citep[sect.\ 3.1]{Ste25mam}.} 
	and the authors immediately correct themselves \citep[pp.\ 65f]{ShaBen14}. In line with the point of the previous section, they rather conclude that (ibid., p.\ 66)
\begin{quote}
there is no inherent generalizability difference between hypotheses. The crucial aspect here is the dependency order between the initial choice of language (or, preference over hypotheses) and the training set. [\dots]\ As long as [this choice of inductive model] is done independently of the training sample, our generalization bound holds.
\end{quote}

We may add that the  bound is model-relative, relative to this initial choice of inductive model. Much like before, the wysiwyg and oracle bounds will be stronger for those hypotheses which have received larger weights, and so will be better in those learning situations in which these preferred hypotheses are in fact good.

Having now discussed the model-relatively in some detail, we must finally confront an obvious question. If the theoretical guarantees are relative to our inductive model, which includes the simplicity ordering, how could we get a non-circular justification for a simplicity preference out of them?


\subsubsection{Circularity?}\label{sssec:simpljustcirc}

Although \citet{BarCevGne22mam} shy away from concluding so explicitly,\footnote{The most Bargagli Stoffi et al.\ conclude, in answer to their ``central question---`is simplicity a road to the truth?{'}'' \citeyearpar[p.\ 21]{BarCevGne22mam}, is that ``the principle of Occam's razor, at least as expressed by introduction of regularization in SRM, can both favor and hamper learning and hence convergence to the truth'' (ibid., p.\ 23).} it is easy to read their work as showing that a simplicity preference is good if the truth is simple, and not good if it is not. Similarly, it seems that \citeauthor{ShaBen14}'s MDL-SRM is a good method if (some) high-weight (``simple'') hypotheses indeed have low true risk, but not so good if they do not.

But then it seems that simplicity here ultimately comes down to a particular domain assumption, which may or may not be appropriate. \citet[ch.\ 6]{Nor21} even argues that talk of simplicity in induction is merely (as \citealp[ch.\ 2]{Sob88} calls it) a ``surrogate'' for specific background facts. According to Norton's material theory \citeyearpar{Nor03pos,Nor21}, inductive inferences are warranted solely by background facts; and in his analysis of curve fitting (\citeyear{Nor21}, sect.\ 6.6), he identifies three such facts.

First, a particular ``error model'' (ibid., p.\ 196), which in our framework presumably includes the assumption of a true data-generating distribution. Second, a parameterization or description of the hypotheses. Norton writes that ``[d]escriptive complexity can only be a good epistemic guide to the truth [\dots]\ if the language of description is chosen so that the truths correspond to simple assertions,'' which in curve fitting means ``a matching with background facts of the parametrization used and the family of its functions from which the curves are drawn'' (ibid., p.\ 199). Third,  an ``order hierarchy'' of families of curves, which ``has to be such that curves fitted earlier in the procedure correspond to stronger or more probable processes'' (ibid., p.\ 202). Thus the choice of ordered sequence of subclasses is only good for induction if it matches the actual domain.\footnote{Norton
	devotes a further chapter to statistical model selection and AIC \citeyearpar[ch.\ 7]{Nor21}. 
	His  discussion of AIC focuses on the material assumption that ``the data are generated by an hypothesis in the model under test'' (ibid., p.\ 205), which does not pertain to statistical learning theory.}

If this is so, then the simplicity preference exhibited by SRM reflects an \emph{ontological} commitment, rather than a \emph{methodological} principle (cf.\ \citealp[ch.\ 2]{Sob88}). It reflects an assumption that ``the world'' (or, in any particular application, the relevant domain) is simple in the corresponding sense, rather than a principle that a simplicity preference is good even without such an assumption. Consequently, if this is so, any justification for simplicity we may obtain from the theory must be question-begging or circular: we can only show a simplicity preference to be good if we first assume that the domain favors simplicity.

I will now argue that this is not so. Even if the theoretical guarantees are model-relative, we can still obtain a justification for a simplicity preference as a methodological principle.

\section{The methodological justification for simplicity}\label{sec:arg}

The original core argument for a simple hypothesis class is that \emph{in order to} profit from the relevant reliability and wysiwgy justification, we must choose a simple class. The main weakness of this means-ends argument and the corresponding methodological norm \eqref{norm:m1} is that it is too restrictive in its applicability.

The new means-ends argument is not so restrictive, because it  deals with a choice of possibly very complex hypothesis class: a countable union of VC classes. It says that \emph{in order to} at least profit from a weaker reliability and wysiwyg justification for such a class, we must carve it up in a sequence of subclasses and implement, as according to norm \eqref{norm:m2}, a certain preference for simplicity to fit in the learning.


\subsection{The argument}
In analogy to the core argument of section \ref{sssec:normn1} above, we can view methodological norm \eqref{norm:m2} as the conclusion of two premises.

\begin{enumerate}
\item In order to have the wysiwyg and reliability guarantees \eqref{eq:srmwysiwyg} and \eqref{eq:srmoracle}, one must trade simplicity for fit.
\item One should aim to have the wysiwyg and reliability guarantees \eqref{eq:srmwysiwyg} and \eqref{eq:srmoracle}.
\end{enumerate}

Notice that this similar argument structure does hide a number of disanalogies with the core argument.  Most importantly, where in the core argument premise 1 is essentially (one direction of) the fundamental theorem, in the current argument premise 1 is more nearly a statement about the working of SRM. 

The formal condition analogous to the choice of a VC class would here be the choice of a countable sequence of VC classes (sects.\ \ref{sssec:just:nonunif}--\ref{sssec:just:srmwysiwyg} above). 
	We could add this condition explicitly to the argument statement, where ``a complex hypothesis class'' is intended to stand for a countable sequence of VC classes.   


\begin{enumerate}
\item In order to have (for a complex hypothesis class) the wysiwyg and reliability guarantees \eqref{eq:srmwysiwyg} and \eqref{eq:srmoracle}, one must trade simplicity for fit.
\item One should aim (for a complex hypothesis class) to have the wysiwyg and reliability guarantees \eqref{eq:srmwysiwyg} and \eqref{eq:srmoracle}.
\end{enumerate}

A more accurate label might be ``possibly complex but not \emph{overly} complex,'' since formally the condition of a countable union of VC classes  includes the case of a very simple (single low-VC-dimension) class, whereas it excludes classes that are not even expressible as such countable unions anymore. But the latter is hardly a restriction (sect.\ \ref{sssec:just:univcons}), and the relevant context for norm \eqref{norm:m2} is where we disregard norm \eqref{norm:m1} because we seek to use a class that is too complex.  For that reason the chosen label makes sense in premise 2, and by extension in premise 1. 

The thrust of premise 1 is that in order to have (for such a complex class) the SRM guarantees, we must use SRM. Spelled out a little more: we must carve up the class in a countable sequence of VC classes, as a precondition to the generalized uniform convergence theorem, specify SRM in accordance with this sequence, and in executing it have it trade fit for simplicity. This may sound unobjectionable, but the formal results I have discussed do not actually \emph{fully} entail it.\footnote{Thanks to an anonymous reviewer for flagging this.} 

Nonuniform learnability implies nonuniform learnability by SRM, but that does not rule out nonuniform learnability by another method; and even for the specific SRM bounds \eqref{eq:srmwysiwyg} and \eqref{eq:srmoracle} I have no result to offer that conclusively rules out that another method could not satisfy alike bounds. However, theorem \ref{thrm:ermfailnonunif} does rule out that ERM, which does not trade fit for simplicity, satisfies these guarantees, and it is hard to imagine a method which similarly profits from the wysiwyg bound given by the generalized uniform convergence theorem (which SRM explicitly minimizes) but looks so different from SRM so as to not trade fit for simplicity.\footnote{Perhaps
	\emph{cross-validation} is such a method. Like the hold-out procedure (fn.\ \ref{fn:holdout} above), this is a common alternative to regularization. However, cross-validation is theoretically not so well-understood \citep[pp.\ 119f]{ShaBen14} and it is not clear that we have something like the guarantees \eqref{eq:srmwysiwyg} and \eqref{eq:srmoracle} for it. I will not pursue this further in the current paper, but if cross-validation can be shown to perform very similarly to SRM and its explicit regularization, this might underwrite a weaker statement of norm \eqref{norm:m2} that we have to proceed ``as if'' we were trading simplicity for fit. Thanks again to an anonymous referee for this suggestion. }
	Thus, while falling short of a watertight mathematical fact, this premise is ultimately a robust-looking observation about SRM and its guarantees. More controversial is again the second premise, stating that these are guarantees one should aim for.


In particular, in light of the model-relativity discussed previously, more needs to be said about why these guarantees are methodologically good, and so why the argument gives a \emph{methodological} justification (section \ref{ssec:argmethod}) for trading simplicity with fit. Relatedly, more needs to be said about why this is a methodological justification for a \emph{simplicity} preference (section \ref{ssec:argsimpl}). I will start with the latter.

\subsection{Simplicity?}\label{ssec:argsimpl}
Recall from section \ref{ssec:srmreg} that the simplicity interpretation stems from the occurrence of the hypothesis subclasses' capacities in the regularization.  

\subsubsection{The weight function}\label{sssec:weights}
However,  there is another variable in the regularization term: the weight of the subclass. In fact, we saw that in the MDL-SRM approach of section \ref{sssec:modelrelweights}, the capacities drop out and \emph{only} the weights matter. Does that mean that simplicity as capacity does not actually need to play a role?

Note, first, that we can also make the weights drop out. Namely, if we have a finite sequence of $(\mathcal{H}_n)_{n<N}$ of hypothesis subclasses, then we can simply use 
the uniform weight function $w: n \mapsto N^{-1}$. The weights can then still be seen to play a role in the generalized uniform convergence theorem: both the reliability and the wysiwyg guarantee feature the additional $-\log N^{-1} = \log N$ term. But since these terms are the same for all hypothesis subclasses (the penalty term rather expresses the size of the superclass), they do not play a role in the regularization.

In fact, for any finite sample size $m$, we will effectively always work with a finite sequence of subclasses. In the usual case of a (nested) sequence of subclasses of increasing VC dimension (like the polynomials), we have that subclasses beyond some $N$ will never be considered, whatever the size-$m$ data, because their penalty is larger than good fit could compensate for. So effectively we are working with a finite sequence again, for which we can imagine a uniform weight function, so that the weights again drop out of the regularization. This kind of reasoning might be why in many---most---presentations of SRM the weights do not appear at all.\footnote{More
	generally (not assuming anything about the VC dimensions), for an infinite sequence  of subclasses, for any (necessarily non-uniform) weight function over the full sequence, the weight terms $w(n)$ must go to zero as $n$ goes to infinity. That means that the $-\log w(n)$ must go to infinity as well, which again means that subclasses beyond some $N$ would never be considered. }

Of course, it is still always possible to adopt non-uniform weights, effectively introducing modified penalty terms. This allows us to encode as inductive bias a further preference for certain hypothesis subclasses over others. 
The MDL-SRM rule is an extreme version of doing so, with different penalties for individual hypotheses. By the previous reasoning, for finite data-size $m$, we are here in a situation where we are effectively working with a \emph{finite} superclass $\mathcal{H}$. Then a choice of uniform weight function  comes down to using the ERM rule with the superclass, while a non-uniform weight function expresses a certain  preference among hypotheses.

The latter might be a good approach in some scenarios, for instance if we have beliefs about the expected performance of different hypotheses.\footnote{Or by more pragmatic \emph{luckiness} reasoning: see section  \ref{sssec:luck} and footnote \ref{fn:luck}.} 
	But special cases of this kind do not show that capacity does not a play a role in regularization by SRM: only that we can engineer it so that all capacities are the same.\footnote{One 
	could think of wackier choices, like designing, for a finite sequence of classes of increasing capacity, a weight function which exactly counteracts the capacities. Again, this does not show that capacities do not play a role in the regularization: only that one could, with some effort, set up things so as to neutralize this effect. } As soon as we make a (more standard) choice for a nested subsequence of increasing capacity, these capacities explicitly appear in the regularization.

\subsubsection{Something else}
Even in the latter standard case, some authors are wary of talking about simplicity. 
\citet[sect.\ 3]{HarKul07} use scare quotes in the title of their section \emph{Induction and ``Simplicity''} on SRM, and consistently talk about SRM as trading off fit against ``something else.'' They indeed ``prefer to
think of VC dimension as providing an \emph{alternative} to simplicity'' \citeyearpar[p.\ 53]{HarKul09abs}.

Their reservations appear to be those already discussed in section \ref{sssec:simpl} above, that capacity  does not necessarily align with other conceptions of simplicity: the ``relevant ordering [\dots]\ is not a simplicity ordering, at least if sine curves count as `simple{'}'' \citep[p.\ 73]{HarKul07}. As we discussed, however, capacity gives \emph{a} natural notion of the simplicity of a hypothesis class. Moreover, the notion of capacity possesses a formal robustness that conceptions deriving from the number of parameters or the ``bumpiness'' of individual hypotheses must lack. 

Yet the worry might persist that it is odd to evaluate simplicity at the level of hypothesis classes, rather than at the level of individual hypotheses. 
What regularization penalty any individual hypothesis is subjected to depends on our initial choice of carving up the superclass in a subclass sequence, and what capacity subclass this hypothesis ends up being in. This, again, need not link to any  notion of the simplicity of the individual hypothesis.

This brings us back to the theme of model-relativity. 
It depends on our initial choice of inductive model how exactly simplicity plays a role in the regularization.  
However, even if this initial choice of subclasses is ours, the subclasses then come attached with a robust notion of their simplicity qua capacity, which automatically plays a role in the SRM regularization. 
It would certainly be nice if we had a further  formal link to the inherent simplicity of individual hypotheses; but we do not. What we do have, or so I will finally argue shortly, is a methodological justification for making a choice of subclass sequence and then regularize accordingly. The regularization term features the capacities of the subclasses, and so there is a trade-off between fit and capacity. If capacity is a natural notion of simplicity, then it is natural to talk about a trade-off between fit and simplicity, and to phrase norm \eqref{norm:m2} in those terms.

\subsection{Methodological?}\label{ssec:argmethod}
A simplicity preference would be methodologically justified, as opposed to merely reflecting an (ontological) assumption  that the world (the domain) is simple, if this simplicity preference can be shown to be beneficial to successful learning, \emph{even without such a simplicity assumption}.

\subsubsection{Not true: clever}
This is exactly what we have in the case of SRM and the  end of (universal) consistency, for uncountable domain (sect.\ \ref{sssec:just:univcons} above). In order to have the guarantee of consistency, it will not do to use ERM: we must employ regularization in accordance with SRM. This is the case even if we do not believe that the domain is simple: the consistency guarantee holds irrespective of whether the Bayes classifier is in a simple class or not. Nor does it matter how we carve up the superclass in a hypothesis class sequence, and so what exact regularization is implemented; but we have to do it somehow, to have the guarantee.

The picture suggested by Norton's material theory, that an inductive method is good if and only if its inductive model matches the material facts (is ``true''), is therefore too coarse. As \citet[p.\ 33]{Gru07} writes, a learning method is 
\begin{quote}
just a \emph{strategy} for inferring models from data (``choose
simple models at small sample sizes''), not a statement about how the world
works (``simple models are more likely to be true'') – indeed, a strategy cannot be true or false, it is ``clever'' or ``stupid.'' And the strategy of preferring
simpler models is clever even if the data-generating process is highly complex [\dots].
\end{quote}

\subsubsection{A pragmatic search strategy}
Various authors have observed that it is clever to adopt a ``search strategy'' of transitioning, as more data comes in, from the simple to the more complex (e.g., \citealp[sect.\ 3]{Kor04mm}; \citealp[sect.\ 3.6]{HarKul07}). What is also pleasing about this picture is its apparent structural similarity to the justification in formal learning theory of a simplicity preference as keeping us on ``the straightest possible path to the truth'' (\citealp{Kel07pos}; cf.\ \citealp{Ste09jpl}).

One might reply, however, that this is more of a \emph{pragmatic} justification for simplicity. The standard pragmatic motivation for simplicity is that we prefer to work with simpler, more convenient theories; the standard example is Mach's view of science as aiming to ``compress'' our experience. \citet[sec.\ 6.3]{Nor21} dismisses this ``simplicity as mere economy of expression,''  as one special case where ``simplicity is sought merely for pragmatic reasons'' (ibid., p.\ 179). But his other special case is precisely the above ``simplicity for economy of search,'' from Popper's conjectures and refutations\footnote{Popper's methodology has also been linked to statistical learning theory \citep{Vap00,CorSchVap09jgps,Ste09jpl}. I will abstain from discussing the merits of this association.} 
	to the justification from formal learning theory.

Norton is right that both cases are weaker than the idea that ``simplicity functions epistemically as a marker of truth; we are to choose the simpler hypothesis or theory because, we are assured, it is more likely to be true'' \citeyearpar[p.\ 178]{Nor21}. A defense of a simplicity preference  as a clever search strategy does also have a pragmatic flavor.\footnote{In 
	the case of formal learning theory, one can say that while finding the truth is an epistemic end, finding it as quickly as possible is already a pragmatic concern (cf.\ \citealp[fn.\ 4]{Str09abs}).} 
	Still, I think it is useful to distinguish the former ``merely'' pragmatic justification (we like simple hypotheses for reasons of convenience) 
	from a \emph{methodological} justification. I will here understand a methodological justification as one which may contain pragmatic elements, but is still tied to advancing the core epistemic end in statistical learning theory, predictive accuracy (low true risk).\footnote{One 
	might reply that the whole approach and framework of statistical machine learning is already thoroughly pragmatic, because of its instrumentalist concern with predictive accuracy as opposed to truth-finding (see, e.g., \citealp[ch.\ 4]{Ots23}). I do not want to deny the difference between truth-finding and predictive accuracy (though see \citealp{Lin24arxiv}), but I still think it is sensible to consider predictive accuracy an epistemic end, and talk about methodological principles as being tied to this end.} 
	Moreover, again, if this is a justification which does not crucially rely on an explicit simplicity assumption, then we can also distinguish it from a ``mere'' ontological principle.

Nevertheless, the above methodological justification of SRM's simplicity preference as a clever search strategy is not fully satisfying. After all, we discussed how (universal) consistency is arguably too weak to be a useful property, and identified the given-sample-size reliability and wysiwyg guarantees as the strongest theoretical justification (sect.\  \ref{ssec:srmthjusts} above). This redirects us from a picture where we have a ``search strategy'' for (unboundedly) growing data-sizes, to one where we need to decide on a good learning method for a given-size dataset.\footnote{Both
	the ontological and the clever search strategy justification for Occam's razor have a long pedigree in the literature. For instance, \citet[sect.\ 9.2.5]{DudHarSto01} give an evolutionary account of the empirical success of Occam's razor, explaining both that  ``we are more likely to ignore problems for which Occam’s razor does not hold'' (restricting ourselves to those where simplicity is a good assumption) and that we are moved towards a pragmatically useful ``design methodology'' of simple to more complex. The reasoning which I will present now in defense of a methodological justification does not, to the best of my knowledge, have a clear precursor.}

\subsubsection{Luckiness}\label{sssec:luck}
But turning to the reliability and wysiwyg guarantees for SRM, it seems more problematic to uphold that  a simplicity preference is good regardless of whether the ``truth'' is simple. After all, in the example inspired by \citet{BarCevGne22mam}, we saw that the SRM guarantees are better or worse than those for ERM depending on  ``the world'' (sect.\ \ref{sssec:modelrelclasses}). 

Nevertheless, there is a reason why, from the perspective of these guarantees, it is clever to use SRM and regularization, even if we do not want to commit to any simplicity assumption. This reason is related to the principle of \emph{luckiness} in MDL inference \citep{Gru07,RooGru11incl}.\footnote{\label{fn:luck}The 
	original idea  maps more directly to a choice of non-uniform weight function: by giving a small number of  classes small weights the reliability bounds for these classes will be significantly better, at the cost of slightly worse bounds for the other classes (see \citealp[p.\ 92]{Gru07}).} 

The idea is that by opting for SRM and regularization (rather than for ERM on the superclass), we stand to gain significantly (if we are lucky), at little cost (if we are not). Specifically, if we are lucky and the true distribution is such that a hypothesis class early in the sequence we chose is good (contains hypotheses with low true risk), then we have a significantly stronger reliability guarantee; whereas if we are unlucky and this is not the case, then the reliability guarantee is only little worse. This effect is illustrated in figure \ref{fig:luck1}, which plots (for standard choice of $\delta = 0.05$ and $m$ in the range 1,000--5,000) the reliability error bounds for the example of linear versus quadratic separators.\footnote{This observation is similar to but not quite the same as the claim that the ``size of the error is about the same as if we had known $n$ beforehand, and minimized the empirical error over $\mathcal{H}_n$'' (\citealp[p.
 294, my notation]{DevGyoLug96}; also see \citealp[p.\ ?]{ShaBen14}). The authors make a comparison between the ERM and the SRM reliability bounds for subclass $\mathcal{H}_n$, whereas the luckiness reasoning is a practically more relevant comparison between the ERM bound for the superclass $\mathcal{H}$ and the SRM bounds for the subclasses $\mathcal{H}_n$. }

\definecolor{clrbl}{RGB}{61,91,153}
\definecolor{clrgr}{RGB}{106,176,76}
\definecolor{clror}{RGB}{255,159,0}
\begin{figure}
    \centering
    \begin{minipage}{0.5\textwidth}
        \centering
        \pgfplotsset{scaled y ticks=false}
        \begin{tikzpicture}
          	\begin{axis} [axis lines=center, xmin = 1000, ymin=0.0275, ymax=0.085, width=6cm, yticklabel style={
        /pgf/number format/fixed}, ytick={0.04,0.06,0.08}, xlabel={$m$}, xtick = {2000,4000}, ylabel = {$\epsilon$},]
    			\addplot [domain=0:5000, smooth, thick, clrbl] {  sqrt((5-(ln(0.05)/ln(10)))/x) }; 			
    			\addplot [domain=0:5000, smooth, thick, clrgr] {  sqrt((4-(ln(0.05)/ln(10))-(ln(0.5)/ln(10)))/x) };
    			\addplot [domain=0:5000, smooth, thick, clror] {  sqrt((5-(ln(0.05)/ln(10))-(ln(0.5)/ln(10)))/x) };
 			\end{axis}
        \end{tikzpicture}
        \vspace{-3mm}
        \captionsetup{width=0.9\textwidth}
        \caption{\small Linear v.\ quadratic. Plotted are the error bound \eqref{eq:ermvsrm1} for ERM (blue) on $\mathcal{H}=\mathcal{H}_2$, and the bounds \eqref{eq:ermvsrm2} for SRM for $\mathcal{H}_1$ (green) and $\mathcal{H}_2$ (red).}\label{fig:luck1}
    \end{minipage}\hfill
    \begin{minipage}{0.5\textwidth}
        \centering
        \vspace{8mm}
        \pgfplotsset{scaled x ticks=false}
        \begin{tikzpicture}
          	\begin{axis} [axis lines=center, xmin = 10000, ymin=0.0, ymax=0.11, width=6cm, xlabel={$m$}, xtick = {20000,40000}, ylabel={$\epsilon$}, yticklabel style={
        /pgf/number format/fixed}, ytick={0.02,0.06,0.1}]
    			\addplot [domain=0:50000, smooth, thick, clrbl] {  sqrt((100-(ln(0.05)/ln(10)))/x) }; 			
    			\addplot [domain=0:50000, smooth, thick, clrgr] {  sqrt((1-(ln(0.05)/ln(10))-(ln(0.01)/ln(10)))/x) };
    			\addplot [domain=0:50000, smooth, thick, clrgr] {  sqrt((25-(ln(0.05)/ln(10))-(ln(0.01)/ln(10)))/x) };
    			\addplot [domain=0:50000, smooth, thick, clrgr] {  sqrt((50-(ln(0.05)/ln(10))-(ln(0.01)/ln(10)))/x) };
    			\addplot [domain=0:50000, smooth, thick, clrgr] {  sqrt((75-(ln(0.05)/ln(10))-(ln(0.01)/ln(10)))/x) };
    			\addplot [domain=0:50000, smooth, thick, clror] {  sqrt((100-(ln(0.05)/ln(10))-(ln(0.01)/ln(10)))/x) };
 			\end{axis}
        \end{tikzpicture}
        \vspace{-3mm}        
        \captionsetup{width=0.9\textwidth}
        \caption{\small Nested sequence $(\mathcal{H}_n)_{n<100}$ with VCdim$(\mathcal{H}_n)=n$. Plotted are the error bound \eqref{eq:ermvsrm1} for ERM on $\mathcal{H}=\mathcal{H}_{100}$ (blue), and the bound \eqref{eq:ermvsrm2} for SRM for $\mathcal{H}_{100}$ (red) and for $\mathcal{H}_1, \mathcal{H}_{25},\mathcal{H}_{50},\mathcal{H}_{75}$ (green). }\label{fig:luck2}
    \end{minipage}
\end{figure}

There are two immediate problems with this idea. First, in our example, the difference between lucky and unlucky does not look \emph{that} significant. But this is not very surprising: that for the (already quite simple) superclass of quadratic separators, a single decomposition gives us a little, but not that much. The effect does get quite significant for  larger classes and decompositions. Figure \ref{fig:luck2} shows (for $\delta=0.05$ and $m$ in the range 10,000--50,000) the reliability bounds for a nested sequence $(\mathcal{H}_n)_{n<N}$ for $N=100$, with VCdim$(\mathcal{H}_n)=n$.\footnote{Here and in the following I assume a uniform weight function $w: n \mapsto 1/N$.} We see that the reliability bounds for SRM are better for by far most subclasses, and many significantly so (e.g., the SRM bound for $\mathcal{H}_{25}$ is about half that of ERM for the superclass).\footnote{In general, for nested sequences $(\mathcal{H}_n)_{n<N}$, the difference is between a term VCdim$(n) - \log N^{-1} = n + \log N$ (in the SRM bound for $\mathcal{H}_n$) and VCdim$(\mathcal{H}_N) = N$ (in the ERM bound). The larger $N$, the more insignificant the $\log N$ compared to the $N$ term, and by definition all $n \leq N$, so for an increasing fraction of $n$ we have a (significantly) better bound. (Though this effect is tempered by the division by sample size $m$.) }

Still, and this is the second problem: if the hypothesis class that we want to work with is too large to give useful reliability bounds for ERM, so that we are drawn to carving up the class and using SRM, the luckiness principle gives little comfort for the possibility  that we are unlucky. It is no good to know that the reliability bound in that unlucky case is not much worse than the reliability bound for ERM, if we already found the ERM bound too weak. To complete the case that SRM is still a clever idea, we need to involve the other component of the theoretical justification, the wysiwyg guarantee.

\subsubsection{What you see}
The wysiwyg guarantees are stronger for the earlier classes in our chosen sequence, whether we are lucky or not. That is, with high probability, for an earlier class in the sequence, the wysiwyg bound  \eqref{eq:srmwysiwyg} is sharp. That means that if the SRM rule selects a hypothesis with low empirical error from an early class, we have some reason to conclude that this hypothesis also has low true error (that we are, in fact, in the lucky case) and/or to act accordingly (to proceed with this hypothesis).\footnote{The
	earlier proviso regarding frequentist guarantees of course still applies (sect.\ \ref{sssec:ermjust2}): strictly speaking, we need some further epistemic or decision-theoretic bridge principle from the guarantee to the particular outcome. }

What if the SRM rule selects a hypothesis with \emph{high} empirical error from an early class? We still have that, for earlier classes in the sequence, the wysiwyg bound  is sharp. So if the empirical error of the selected hypotheses is high, we have some reason to conclude that the true error of the hypothesis and so indeed the hypothesis class is also bad (that we are, in fact, in the unlucky case) and/or to act accordingly (to go back to the drawing board). In either case, the wysiwyg guarantee tells us that, with high probability, we have an indication whether we are on the right track, which allows us to act accordingly.

This, of course, still leaves the possibility that, despite the larger penalty term, SRM selects a hypothesis far off in the sequence, with weak wysiwyg bounds. But in that case, we in principle have this information: that SRM selected from one of the far-off classes. That in itself strongly indicates that our inductive model was off: that we are not in the lucky case.\footnote{Here
	the (informal) reasoning is that this would be \emph{very} unlikely to happen if we are, in fact, in the lucky case, because it means that the data is not only such that the good classes do bad, but so bad that the additional penalty for the bad classes is overcome. The reasoning is somewhat informal because by assuming the lucky case we are assuming some subset of possible true distributions, and precise probabilistic statements would depend on what this subset is.} 
	Hence, again, we have reason to act accordingly: to go back to the drawing board.

\subsubsection{Pragmatic, ontological, methodological}
There are clearly pragmatic elements to the previous reasoning. The wysiwyg justification is pragmatic rather than epistemic, insofar as it is not a guarantee of accurate prediction, but of having an indication of this accuracy, which derives its importance from the need to make a decision what to do next. The reliability guarantee is a better candidate for a purely epistemic justification, but as we saw it needs the wysiwyg to complete the justification for \eqref{norm:m2}. In any case, similarly to norm \eqref{norm:m1}, the actual implementation of the norm would involve various further considerations and decisions. Standing out here are of course the decision that \eqref{norm:m1} is infeasible, and the choice of hypothesis class sequence in implementing SRM.

Moreover, an ``ontic'' element remains in that the reliability guarantees are stronger if the hypothesis sequence is a better match with the ``truth.'' Insofar as this choice of inductive model is a reflection of our prior knowledge, it constitutes a better inductive bias if our prior knowledge is indeed accurate.  

Nevertheless, the reasoning still underpins a methodological norm, which neither collapses to a purely pragmatic simplicity preference nor to a purely ontic simplicity assumption. It is not purely pragmatic because it is still tied to the epistemic end of predictive accuracy: the means-ends argument for \eqref{norm:m2} says that in order to have reliability (good accuracy) and wysiwyg (good indication of accuracy) guarantees, it is a good idea to regularize. It is not a purely ontic assumption, because regularization is still a good idea even if we are not sure our chosen hypothesis sequence and corresponding regularization terms reflect the domain well.\footnote{Indeed, the reasoning shows that it is still a good idea to use regularization if we are given some default ordering of classes, or the ordering is implicit in the adopted regularization procedure.}


\section{Conclusion}\label{sec:concl}
It is important to be clear about the limitations of the argument. Since it is embedded in statistical learning theory, the argument is confined to the framework's idiosyncrasies: its exclusive concern with predictive accuracy, under the assumption of a stable underlying probability distribution.  The corresponding frequentist nature of the theoretical justification raises its interpretational questions, and the reasoning essentially relies on randomness in the data; in short, this is certainly not the full story about simplicity in science \citep[sect.\ 3]{Kel11incl}. The model-relativity of the theoretical justification, together with a formal notion of simplicity that attaches to classes rather than individual hypotheses, further results in a means-ends reasoning that is less straightforward than earlier attempts to defend Occam's razor in machine learning---but as such also does justice to critiques of these attempts (e.g., \citealp{Dom99dmkd}; also see again \citealp{Ste25mam}). Finally, I am not making claims about the extent to which the argument applies to regularization techniques, like drop-out or early stopping in deep learning, which are theoretically less well understood. But I do claim that the argument gives a methodological, non-circular justification for the norm of trading off fit for simplicity---and as such does justice to a central methodological practice in machine learning.

Or have this norm and practice been overtaken by events too? The contemporary debate about ``benign interpolation'' and ``double descent'' essentially revolves around the observation that not just norm \eqref{norm:m1} but also norm \eqref{norm:m2} no longer seem (as) relevant: we achieve good learning with extremely complex classes and no regularization \citep{Bel21ac,BarMonRak21ac}. A philosophical appraisal of the consequences of this observation, and in particular of the renewed evocation of a principle of Occam's razor at the center of a suitably improved theory of generalization, must be postponed to further work. But an important precondition for such work is clarity on the principle's role in the classical theory, as offered here.





\small

\appendix
\section*{Appendix}\label{appx}

\begin{thrm}\label{thrm:ermfailnonunif}
ERM fails to nonuniformly learn any non-VC hypothesis class $\mathcal{H}$.  
\end{thrm}
\begin{proof}
Choose any 
reference $h^\circ \in \mathcal{H}$. We show that for any large enough sample size $m$, there exists a true distribution $\mathcal{D}_m^*$ such that
\begin{align}
 \mathbb{P}_{S \sim (\mathcal{D}_m^*)^m}\left[L_{\mathcal{D}_m^*}(\mathrm{ERM}_\mathcal{H}(S)) -  L_{\mathcal{D}_m^* }(h^\circ) \geq \frac{1}{48} \right] > \frac{1}{45},
\end{align} 
which refutes nonuniform learnability.

Given $m$, choose a subset $C = \{x_1, \dots, x_{2m} \} \subset \mathcal{X}$ of $2m$ instances which is shattered by $h^\circ$ (such a set must exist since $\mathcal{H}$ has infinite VC dimension). Define $\mathcal{D}_m^*$ by 
\begin{align}
\mathcal{D}_m^*(\{x,y\})= \begin{cases}
  \frac{2}{6m} & \textrm{if } x \in C \textrm{ and  } y = h^\circ(x) \\
  \frac{1}{6m} & \textrm{if } x \in C \textrm{ and  } y \neq h^\circ(x) \\
  0 & \textrm{otherwise}.
\end{cases}
\end{align}
In words, $\mathcal{D}_m^*$ assigns uniform probability $(2m)^{-1}$ to all and only the instances in $C$, and with probability two-thirds labels the instance according to $h^\circ$. That is, under this distribution, $h^\circ$ is the signal but with substantial noise;  the proof idea is that ERM will with a certain probability overfit on noise, driving a wedge between its and $h^\circ$'s true risk.

We first show that, for large enough $m$, if we sample a size-$m$ sequence $S$ from $\mathcal{D}_m^*$, then with probability greater than one-fifth at least half of the instances $x$ in $S$ (so disregarding the labels) are only sampled once. Call such instances that only appear once \emph{singletons}. Let  random variable $Y$ denote the proportion of such singletons in $S$, and let random variable $X_i$ take value 1 if $x_i$ is a singleton and 0 otherwise. Utilizing the linearity of expectation, 
\begin{align}
\Expec_{S \sim (\mathcal{D}_m^*)^m} \left[ Y \right]  &=   \Expec_{S \sim (\mathcal{D}_m^*)^m} \left[ \frac{1}{m} \sum_{i=1}^{2m} X_i \right] \\
&= \frac{1}{m} \sum_{i=1}^{2m} \Expec_{S \sim (\mathcal{D}_m^*)^m} \left[  X_i \right] \label{eq:linexp}\\
&= \frac{1}{m} \sum_{i=1}^{2m} \binom{m}{1} \frac{1}{2m}\left(1-\frac{1}{2m} \right)^{m-1} \\
&=  \left(1-\frac{1}{2m} \right)^{m-1} 
>  \left(\left(1-\frac{1}{2m} \right)^{2m}\right)^{\frac{1}{2}}.
\end{align} 
By the standard limit $\lim_{x \rightarrow \infty} (1-x^{-1})^x = e^{-1} $ the last term converges for increasing $m$ to $e^{-1/2} \approx 0.607$, so that for sufficiently large $m$ we have  
\begin{align}
\Expec_{S \sim (\mathcal{D}_m^*)^m} \left[ Y \right]  > \frac{6}{10}.
\end{align} 
By Markov's inequality (see \citealp[app.\ B1]{ShaBen14}),
\begin{align}
\mathbb{P}_{S \sim (\mathcal{D}_m^*)^m}\left[ Y \geq \frac{1}{2} \right] &\geq \frac{\Expec_{S \sim (\mathcal{D}_m^*)^m} \left[ Y \right] - \frac{1}{2}}{1-\frac{1}{2}}   
> \frac{\frac{6}{10}- \frac{1}{2}}{1-\frac{1}{2}} = \frac{1}{5}. 
\end{align} 

Further, let $Z$ denote the proportion of sampled singleton instances that are  \emph{nonrepresentative}: that receive a different label than prescribed by $h^\circ$. Clearly, by definition of $\mathcal{D}_m^*$, for any subsequence of instances $x$ sampled from $(\mathcal{D}_m^*)^m$, the expected proportion of those instances that are subsequently labeled in a nonrepresentative manner equals one-third. In particular,
\begin{align}
\Expec_{S \sim (\mathcal{D}_m^*)^m} \left[ Z \right]  &= \frac{1}{3},   
\end{align} 
so that, again by Markov's inequality, the probability of at least a quarter of nonrepresentative instances among the singleton ones  is 
\begin{align}
\mathbb{P}_{S \sim (\mathcal{D}_m^*)^m}\left[ Z \geq \frac{1}{4} \right] &\geq \frac{\Expec_{S \sim (\mathcal{D}_m^*)^m} \left[ Z \right] - \frac{1}{4}}{1-\frac{1}{4}}  = \frac{1}{9}. 
\end{align} 
In sum, the probability of sampling a proportion of at least one-half singleton instances among which at least a quarter are non-representative is
\begin{align}
\mathbb{P}_{S \sim (\mathcal{D}_m^*)^m}\left[ Y \geq \frac{1}{2} \ \& \ Z \geq \frac{1}{4} \right] > \frac{1}{5} \cdot \frac{1}{9}=\frac{1}{45}. 
\end{align} 

Finally, if $S$ is such that it contains at least $m/8$ non-representative singleton instances, then ERM must also classify those instances in a non-representative manner. Since
\begin{align}
L_{\mathcal{D}_m^* }(h^\circ) = \frac{1}{3},
\end{align} 
we therefore have for such $S$ that
\begin{align}
L_{\mathcal{D}_m^*}(\mathrm{ERM}_\mathcal{H}(S)) -  L_{\mathcal{D}_m^* }(h^\circ) &= L_{\mathcal{D}_m^*}(\mathrm{ERM}_\mathcal{H}(S)) - \frac{1}{3} \\
 &\geq  \left( \frac{1}{16} \cdot \frac{2}{3} + \frac{15}{16} \cdot \frac{1}{3} \right)  - \left( \frac{1}{16} \cdot \frac{1}{3} + \frac{15}{16} \cdot \frac{1}{3}  \right) = \frac{1}{48}, 
\end{align} 
which concludes the proof.
\end{proof}

\end{document}